%% file: arxiv.tex
\documentclass{article}
\usepackage[margin=1in]{geometry}

\usepackage{microtype}
\usepackage{graphicx}
\usepackage{authblk}
\usepackage{subcaption}
\usepackage{natbib}
\usepackage{booktabs} 
\usepackage{etoc}
\usepackage{titletoc}
\usepackage{parskip}
\usepackage{hyperref}
\hypersetup{colorlinks=true, linkcolor=blue, citecolor=blue, urlcolor=blue}

\usepackage{amsmath}
\usepackage{amssymb}
\usepackage{mathtools}
\usepackage{amsthm}

\usepackage[capitalize,noabbrev]{cleveref}

\theoremstyle{plain}
\newtheorem{theorem}{Theorem}[section]

\newtheorem{lemma}[theorem]{Lemma}

\theoremstyle{definition}
\newtheorem{definition}[theorem]{Definition}

\theoremstyle{remark}

\usepackage[textsize=tiny]{todonotes}
\usepackage{tcolorbox}
\usepackage{url}
\usepackage{amsthm}
\usepackage{subcaption}
\usepackage{cancel}

\definecolor{mgh}{rgb}{0.1,0.5,0.1}

\definecolor{lu}{rgb}{0.8,0.1,0.1}

\definecolor{gr}{rgb}{0.4,0.7,0.2}

\definecolor{st}{rgb}{0.1,0.5,0.1}

\definecolor{greenp}{rgb}{0.0, 0.51, 0.5}

\newcommand{\argmin}{\mathop{\rm argmin}}
\newcommand{\argmax}{\mathop{\rm argmax}}

\input{macros}
\title{\textbf{Provably avoiding over-optimization in Direct Preference Optimization without knowing the data distribution}}

\author[1]{Adam Barla\thanks{Equal contributing authors ordered alphabetically.}}
\author[1]{Emanuele Nevali\protect\footnotemark[1]}
\author[1]{Luca Viano\protect\footnotemark[1] \thanks{Correspondence to \texttt{luca.viano@epfl.ch}}}
\author[1]{Volkan Cevher}

\affil[1]{EPFL}

\begin{document}

\maketitle
\begin{abstract}
We introduce \texttt{PEPO} (\emph{Pessimistic Ensemble based Preference Optimization}), a single-step Direct Preference Optimization (DPO)-like algorithm to mitigate the well-known over-optimization issue in preference learning without requiring the knowledge of the data-generating distribution  $\pidata$ or learning an explicit reward model.   \texttt{PEPO} achieves pessimism via an ensemble of preference-optimized policies trained on disjoint data subsets and then aggregates them through a worst case construction that favors the agreement across models. In the tabular setting, \texttt{PEPO} achieves sample complexity guarantees depending only on a single-policy concentrability coefficient, thus avoiding the all-policy concentrability which affects the guarantees of algorithms prone to over-optimization, such as DPO.  The theoretical findings are corroborated by a convincing practical performance, while retaining the simplicity and the practicality of DPO-style training.
\end{abstract}


\section{Introduction}

Learning from binary preferences is a fundamental technique for the post-training of large language models \citep{christiano2017deep,ouyang2022training}. Among existing approaches, \emph{Direct Preference Optimization} (DPO) \citep{rafailov2023direct} has gained widespread adoption due to its strong empirical performance and simplicity of implementation: indeed, DPO elegantly bypasses the intermediate step of learning a reward function from the preference dataset. 

Despite these advantages, DPO is known to suffer from \emph{over-optimization}: as training progresses, improvements in optimizing the training loss do not necessarily translate into better generation quality, and may even degrade performance relative to the initial model \citep{gao2023scaling, park2024disentangling, xu2024is}.
Moreover, early stopping using the validation loss as a criterion is not reliable because it is only a loose proxy for what really matters, i.e. the quality of the responses that the model can generate in practice.

Several new DPO variants 
mitigate over-optimization via \emph{pessimistic regularization}. Notably, approaches such as $\chi^2$-regularized preference optimization ($\chi^2$PO) \citep{huang2024correcting} and DPO augmented with supervised fine-tuning (DPO+SFT) \citep{liu2024provably} achieve guarantees that depend only on a \emph{single-policy} concentrability coefficient, which is the theoretical certificate of robustness to over-optimization \citep{huang2024correcting}. Indeed, these guarantees imply that the learned policy is competitive with any policy sufficiently covered by the data. In stark contrast, non-pessimistic approaches suffer from an \emph{all-policy} concentrability coefficient, meaning that the output policy can be significantly suboptimal even compared to well-covered policies unless the dataset covers the whole policy class.  

A notable merit of the guarantees in \citep{huang2024correcting,liu2024provably} is that they apply to general function classes, even to non-convex ones in the case of $\chi^2$PO. However, these guarantees rely on the assumption that we have access to the data-generating distribution $\pidata$, or to a reference/base policy that is well-covered by $\pidata$. In practice, this assumption is often violated, for instance, when smaller models are post-trained using preference data generated by larger, proprietary systems whose logits can not be accessed when fine-tuning the smaller models. A common example is the UltraFeedback dataset  \citep{cui2023ultrafeedback} in which the responses are generated via a pool of LLMs, including proprietary ones such as  GPT-4. 

This gap motivates the following research question:


\begin{tcolorbox}[
  colback=blue!10!white,
  colframe=blue!80!black,
  width=\linewidth,
  before skip=10pt, 
  after skip=10pt   
]
\textbf{\textcolor{blue}{Open Question:}} Can we design a DPO-like algorithm (i.e., which avoids learning an explicit reward) that is provably robust to over-optimization without knowledge of data-generating distribution $\pidata$?
\end{tcolorbox}

Our work answers this question affirmatively by introducing \texttt{PEPO} (\emph{Pessimistic Ensemble based Preference Optimization}). The central idea behind \texttt{PEPO} is to resort to \emph{pessimism via ensembles}: a technique as simple as unexplored in the context of LLM alignment, especially in the DPO setting.

Concretely, \texttt{PEPO} trains an ensemble of lightweight DPO-style policies on disjoint subsets of the preference dataset, implemented via low-rank adapters \citep{hu2022lora}. Finally, at generation time, we use a certain notion of ensemble agreement to generate responses either via rejection sampling or approximating the (otherwise intractable) response-level normalization constant to the token-level counterpart.

\paragraph{Theoretical contribution} Our main theoretical result shows that, in the tabular setting, \texttt{PEPO} achieves sample complexity guarantees depending only on a single-policy concentrability coefficient, despite making no assumptions about $\pi_{\text{data}}$ and not learning a reward model. This contrasts with standard DPO \citep{rafailov2023direct}, which inherently incurs in an all-policy concentrability dependence, with existing pessimistic DPO variants that require access to the data-generating distribution \citep{huang2024correcting, liu2024provably} and with purely theoretical works which are not DPO-like because they require a reward modeling step \citep{zhu2023principled,zhan2023provable,schlaginhaufen2025efficient,li2023reinforcement}. Moreover, we characterize the optimal \texttt{PEPO} ensemble size: it is the smallest value that is sufficiently high to guarantee pessimism.

\paragraph{Practical contribution} Empirically we corroborate our analysis, demonstrating that \texttt{PEPO} consistently improves over DPO in post-training the following 7/8B parameters models: Zephyr-7B\footnote{\url{https://huggingface.co/alignment-handbook/zephyr-7b-sft-full}}, Llama-3.1-8B\footnote{\url{https://huggingface.co/allenai/Llama-3.1-Tulu-3-8B-SFT}}, and Mistral-7B\footnote{\url{https://huggingface.co/HuggingFaceH4/mistral-7b-sft-beta}} and even at 34B scale (Yi-34B \footnote{\url{https://huggingface.co/01-ai/Yi-34B}}), where \texttt{PEPO} still performs convincingly.

\paragraph{Paper Organization} The next section clarifies the setting and the notation used in this work.  \Cref{sec:algo} presents in greater detail our algorithm, i.e. $\texttt{PEPO}$. Then, \Cref{sec:theory} formally develops the guarantees for PEPO, 
and, finally  \Cref{sec:experiments} demonstrates its practical performance.

\section{Preliminaries}
In \emph{learning from preferences}, we consider a prompt/state space $\X$ and a response/action space $\A$, the learning algorithm receives as input a starting model $\piref: \X \rightarrow \Delta_{\A}$ which is usually a model that exits a \emph{Supervised Fine-Tuning} (SFT) routine, and a preference dataset of the form $\mathcal{D} = \bcc{X_n, A^+_n, A^-_n}^N_{n=1}$ where $X_n \in \X, A^+_n, A^-_n \in \A\times\A$ for all $n\in [N]$. In the dataset $\mathcal{D}$, we observe the initial prompts sampled from an initial distribution $\initial$, i.e. $X_n \sim \initial$ for each $n \in [N]$. Intuitively, $\initial(x)$ is the probability with which a human presents prompt $x$ to the language model.
Moreover, for each $X_n$, $\mathcal{D}$ contains two responses $A^+_n, A^-_n$ sampled from an \emph{unknown} distribution $\pidata$. 
Moreover, the action $A^+_n$ has been judged the winner over $A^-_n$ by a Bradley-Terry preference model defined as follows.
\begin{definition}\textbf{Bradley-Terry (BT) preference model }
Given  a prompt $x\in \X$, two responses $a,b \in \A\times\A$ and access to the \emph{ground truth} reward function $r^\star :\X\times\A \rightarrow [0, \RMAX] $ for some $\RMAX \in \mathbb{R}$, the Bradley-Terry preference model computes the probability of $a$ being a winner over $b$  in response to the prompt $x$ (event denoted as $a \succ b | x$ ) as $
\mathbb{P}^\mathrm{true}_{r^\star}(a\succ b | x) = \sigma(\Delta_{r^\star}(x,a,b))$
where $\Delta_{\mathfrak{f}}(x,a,b) = \mathfrak{f}(x,a) - \mathfrak{f}(x,b) $ for any function $\mathfrak{f}:\X\times\A \rightarrow \mathbb{R}$ and $\sigma(x) = \frac{1}{1 + e^{-x}}$ denotes the sigmoid function. 

Given the BT model above, the dataset is labeled sampling $Z \sim \mathrm{Bernoulli}(\mathbb{P}^\mathrm{true}_{r^\star}(a\succ b | x))$ for each $x,a,b$ appearing in the dataset. If $Z=1$, $a$ is declared as the winner; vice versa if $Z=0$, $b$ wins the comparison.
\end{definition}
The goal of the learning algorithm is to update the starting model $\piref$ to obtain a post-trained model $\piout$ such that it approximately equals
\begin{equation}
\label{eq:goal}
\argmax_{\pi\in\Pi} J_\beta(\pi) := \innerprod{\pi}{r^\star} - \beta KL(\pi,\piref)
\end{equation}
where $\Pi$ is the set of all mappings from prompts to distributions over responses $\Pi = \bcc{\pi: \X\rightarrow\Delta_{\A}}$. The inner product notation is defined for each policy $\pi \in \Pi$ and for any function $r:\X\times\A \rightarrow [0, \RMAX]$ as $\innerprod{\pi}{r}:= \sum_{x,a\in\X\times\A} \initial(x) \pi(a|x) r(x,a)$ and the average (over prompts) KL is defined for any policy pair $\pi,\pi'\in \Pi\times\Pi$ as $KL(\pi,\pi') = \sum_{x,a\in\X\times\A} \initial(x) \pi(a|x) \log \nicefrac{\pi(a|x)}{\pi'(a|x)}$.

The objective in \eqref{eq:goal} therefore prescribes to find a trade-off between two desiderata: (i) maximizing the \emph{ground truth} reward $r^\star$ used by the Bradley-Terry preference model to label the dataset and (ii) minimizing the shift as measured by the KL divergence from the initial model $\piref$. The value of $\beta \in (0, \infty)$ must be set as high as much it is important to stay close to $\piref$ to make sure that $\piout$ inherits the useful properties learned by $\piref$ from the data observed before entering the learning from preferences phase. 

We now describe two fundamental ways to solve \eqref{eq:goal}, which will serve as building blocks for our approach.

\paragraph{Reinforcement Learning from Human Feedback: RLHF}
In RLHF, the idea is to optimize \eqref{eq:goal} under two approximations. First, we approximate $r^\star$, which is unknown to the learner, with the maximum likelihood estimator, denoted by $\hat{r}$ and learned on the dataset $\mathcal{D}$ as follows:
\begin{equation}
\label{eq:rMLE}
\hat{r} = \argmax_{r: \X\times\A\rightarrow[0,\RMAX]} \sum^N_{n=1} \log \sigma (\Delta_r(X_n, A^+_n, A^-_n)).
\end{equation}
Secondly, in \eqref{eq:goal}, the prompt distribution $\initial$ is unknown; therefore, it must be approximated using, for example, the states observed in the dataset. All in all, RLHF methods output the policy maximizing the following empirical objective over $\pi\in \Pi$,
\begin{equation}
\label{eq:optPolicyforrMLE}\sum^N_{n=1}\sum_{a\in\A} \pi(a|X_n) \brr{\hat{r}(X_n,a) - \beta \log \frac{\pi(a|X_n)}{\piref(a|X_n)}}.   
\end{equation}
RLHF has the merit of being the conceptual core of the large majority of, if not all, alignment algorithms. Sadly, it has the drawback of being split into two training steps represented by equations \eqref{eq:rMLE} and \eqref{eq:optPolicyforrMLE} respectively. This has been elegantly avoided by DPO \citep{rafailov2023direct} with a reparameterization trick explained next.
\paragraph{Direct Preference Optimization}
The core observation is that \eqref{eq:goal}, for any reward function $r$, admits the following closed form solution
\[
\pi_r(a|x) = \frac{\piref(a|x) \exp(\frac{r(x,a)}{\beta})}{Z_r(x)},
\]
where $ Z_r(x) := \sum_{b} \piref(b|x) \exp(\frac{r(x,b)}{\beta})$. Albeit, being incomputable due to the hardness of computing  $Z_r(x)$ (unless the reward takes a very specific form \citep{matrenok2025quantile}), the closed form expression for $\pi_r$ allows to write the reward function as $$r(x,a) = \beta \log \nicefrac{\pi_r(a|x)}{\piref(a|x)} + \beta \log Z_r(x) \label{eq:rfrompi}.$$ Then, replacing this expression in \eqref{eq:rMLE}, we can ensure that $\pi_{\hat{r}}$ is in the solution set of the following single-step optimization problem, which is known as the DPO objective,
\begin{equation}
\label{eq:DPO}
    \pi_{\hat{r}} \in \argmax_{\pi \in \Pi} \sum^N_{n=1} \log \sigma (\beta \Delta_{\log \pi/\piref}(X_n, A^+_n, A^-_n))).
\end{equation}
Importantly, we have used that for each reward $r$,  $\beta \Delta_{\log \pi_r/\piref}(x,a,b) := \beta \log \frac{\pi_r(a|x)}{\piref(a|x)} - \beta \log \frac{\pi_r(b|x)}{\piref(b|x)} $ equals $\Delta_r(x,a,b)$ because the normalization constant $Z(x)$ cancels out.

\paragraph{Limitations of RLHF and DPO (i.e., the lack of pessimism):} The missing ingredient above is that both method maximizes $\hat{r}$ (directly or indirectly) blindly, i.e., treating it as if it were a perfect estimate of the ground truth $r^\star$. However, in practical situations where $N$ is finite, this is not the case: $\hat{r}$ might not be accurate, and taking into account its uncertainty becomes fundamentally important. As an example, think of a single state, two actions case in which $\hat{r}(a_1) = \hat{r}(a_2) + \epsilon$ for some small $\epsilon$ but the uncertainty\footnote{We keep uncertainty vague here for the sake of intuition.} of $\hat{r}(a_1)$ is much higher than $\hat{r}(a_2)$. Then, choosing $a_2$ is more reasonable while RHLF and DPO, which purely maximize $\hat{r}$ would output the highly risky action $a_1$.
In a less toy example targeted to LLMs, DPO and RLHF might achieve high reward under $\hat{r}$ exploiting hallucinations of $\hat{r}$, i.e., actions that have in fact small reward under the ground truth $r^\star$.

To mitigate these issues, the next section presents our techniques to incorporate pessimism in DPO. In contrast  to previous pessimistic DPO versions \citep{huang2024correcting,liu2024provably}, we \emph{assume no knowledge at all} of $\pidata$ whereas \citep{huang2024correcting} assumed that $\pidata = \piref$ and \citep{liu2024provably} leverages access to a policy $\pi_{\mathrm{base}}$ such that $\norm{\nicefrac{\pi_{\mathrm{base}}}{\pidata}}_{\infty}$ is small. 
\section{The algorithm}
\label{sec:algo}
Our algorithm, \texttt{PEPO}, applies  three important changes on top of DPO which will enable our guarantees featuring only the \emph{single-policy} concentrability coefficient: 
\begin{itemize}
\item We split the dataset in $L$ subsets of equal size and we maximize $L$ different instances of the pessimistic DPO objective with respect to independent LoRA parameters over each subset.
\item Finally, our theory motivates outputting a policy obtained aggregating the ensemble in a pessimistic manner. Unfortunately, computing such policy is intractable but we present a rejection sampling method to sample exactly from it. 
\end{itemize}
We now move to a detailed overview of these two components whose interplay is key for the theoretical guarantees presented in the next section. 
\begin{algorithm}[!t]
\caption{\texttt{PEPO} \label{alg:pepo}}
    \begin{algorithmic}[1]
    \STATE \textbf{Require} number of ensembles, i.e.  $L$.
    \STATE 
    \textcolor{orange}{\% Train time.}
    \STATE Create dataset partitions $\cup_{\ell\in[L]}\mathcal{D}^\ell = \mathcal{D}$.
        \FOR{$\ell \in [L]$} 
\STATE \[
\tilde{\pi}^\ell \in \argmax_{\pi\in\Pi} J_{\mathrm{pessDPO}}(\pi; \mathcal{D}^\ell) 
\]
\ENDFOR
\STATE Define the output distribution
\[
\piout(a|x) = \frac{\min_{\ell\in[L]}\tilde{\pi}^{\ell}(a|x) \exp\br{-\zeta^\ell(x)}}{\sum_{a\in\A}\min_{\ell\in[L]}\tilde{\pi}^{\ell}(a|x) \exp\br{-\zeta^\ell(x)}}
\]
\STATE \textcolor{orange}{\% Test time.}
\STATE \textbf{Require} User prompt $x_{\mathrm{test}} \in \X$
\STATE Invoke \Cref{alg:rejsampl} to sample $A \sim \piout(\cdot|x_{\mathrm{test}})$.
\end{algorithmic}
\end{algorithm}
\subsection{Creating the models ensemble}
We build on the pessimistic sigmoid function to train $L$ pessimistic language models over a partition of the preference dataset. To this end, we partition the data $\mathcal{D}=\bcc{X_n, A^+_n, A^-_n}^N_{n=1}$ in $L$ subsets, i.e., $\mathcal{D} = \cup^L_{\ell=1}\mathcal{D}^{\ell}$,
where we denote the data in the $\ell^{\mathrm{th}}$ chunk as $\mathcal{D}^{\ell} = \bcc{X_{n,\ell}, A^+_{n,\ell}, A^-_{n,\ell}}$.
At this point, we compute $\tilde{\pi}^\ell$ for each $\ell \in [L]$ as the solution of the following pessimistic DPO objective which modifies the original DPO objective in \eqref{eq:DPO} by using only the subset $\mathcal{D}^\ell$ instead of using the full dataset $\mathcal{D}$ and replacing the standard sigmoid with one horizontally shifted by $\lambda^\ell$, i.e. 
\begin{align}
\tilde{\pi}^\ell &\in \argmax_{\pi\in\overline{\Pi}} J_{\mathrm{pessDPO}}(\pi; \mathcal{D}^\ell)  := \sum_{x,a,b \in \mathcal{D}^\ell} \log \br{\sigma\br{\beta\Delta_{\log \pi/\piref}(x,a,b) + \lambda^\ell(x,a,b)}}, 
\label{eq:RegDPO} 
\end{align}
where $\lambda^\ell > 0$ parameter serves to reduce (slightly) the margin $\beta\Delta_{\log \pi/\piref}(x,a,b)$ necessary to induce a certain winning probability. Vice versa, the estimated reward margin between the winning and losing responses for a given winning probability is estimated more conservatively; it is smaller than the one that standard DPO would assign. The more the triplet $x,a,b$ feels uncertain for the model $\tilde{\pi}^\ell$, the larger $\lambda^\ell(x,a,b)$ becomes. Having a dependence on the ensemble index thus allows us to instantiate models with different levels of confidence for different regions of the state action space.
Moreover, we maximize \eqref{eq:RegDPO} over a subset of $\Pi$ denoted by $\overline{\Pi}$, defined as $\overline{\Pi} := \bcc{\pi \in \Pi : \abs{\beta \Delta_{\log \pi/\piref}} \leq 2\RMAX}$. This restriction is important in the analysis because it ensures that $\beta\Delta_{\log \tilde{\pi}^\ell/\piref}$ is bounded for each $\ell \in [L]$.
\subsection{Pessimism via ensemble}
We now aim at aggregating the $L$ models in the ensemble in a pessimistic manner. Towards this goal, recall that by the correspondence between policy and reward space given by \eqref{eq:rfrompi}, each implicit reward learned obtained as $\beta \log \frac{\tilde{\pi}^{\ell}(a|x)}{\piref(a|x)} - N^{-\ell}(x)\beta\sum_{X,A \in \cD\setminus\cD^\ell} \log \frac{\tilde{\pi}^\ell(A|X)}{\piref(A|X)}\mathds{1}\bcc{X = x}$ can be considered as an independent estimate of the ground truth reward, assuming without loss of generality that $\sum_b \pidata(b|x)r^\star(x,b) = 0$.

At this point, similarly to \cite{coste2023reward}, we aim at optimizing a notion of worst case reward over the ensemble. In our case, we choose  the worst case reward as $$ r^-(x,a) = \beta \min_{\ell \in [L]}  \br{\log \frac{\tilde{\pi}^{\ell}(a|x)}{\piref(a|x)} -  N^{-\ell}(x)\sum_{X,A \in \cD\setminus\cD^\ell} \log \frac{\tilde{\pi}^\ell(A|X)}{\piref(A|X)}\mathds{1}\bcc{X = x} }, 
\label{eq:worst_case_reward}$$ we will also denote $\zeta^\ell(x) = N^{-\ell}(x) \sum_{X,A \in \cD\setminus\cD^\ell} \log \frac{\tilde{\pi}^\ell(A|X)}{\piref(A|X)}\mathds{1}\bcc{X = x}  $ and we use the notation $N^{-\ell}(x)$ to indicate how many times the state $x$ appeared in the dataset $\cD\setminus \cD^\ell$. 
Then, we define our output policy as the policy learned by the following \emph{virtual} RLHF problem:
\begin{equation}
\piout = \argmax_{\pi\in \Pi} \innerprod{\pi}{r^-} - \beta D_{\mathrm{KL}}(\pi,\piref) .\label{eq:outputDPO}
\end{equation}
We dubbed the above problem as virtual because the program in \eqref{eq:outputDPO} never needs to be actually implemented in practice. Indeed, we will show next how to efficiently sample from $\piout$ via rejection sampling, thus bypassing the need for computing $\piout$ explicitly. 
\begin{algorithm}[!t]
\caption{\texttt{Efficient Trajectory Sampling via Rejection Sampling} \label{alg:rejsampl}}
    \begin{algorithmic}[1]
    \STATE \textbf{Require} Initial Prompt $x$, Nominator of $\piout$
    \[
    f_{\mathrm{out}}(x,a) = \min_{\ell\in[L]}\tilde{\pi}^{\ell}(a|x) e^{-\zeta^\ell(x)} 
    \]
    \STATE \textbf{Require:} Number of trials $N_{\RS} = \frac{\log(1/\delta)}{\sum_{a\in\A} f_{\mathrm{out}}(x,a)}$
    \FOR{$n=1, \dots, N_{\RS}$}
\STATE Sample $A^n \sim \pi_{\mathrm{prop}}(\cdot|x)$.
\STATE Compute  $\alpha = \frac{f_{\mathrm{out}}(x,A^n)}{\pi_{\mathrm{prop}}(A^n|x)}$ \textcolor{orange}{\% Notice that $\alpha \in [0,1]$.}
\STATE Sample $U \sim \mathrm{Unif}([0,1])$.
\STATE If $U \leq \alpha$, output $A^n$ otherwise continue.
    \ENDFOR
    \end{algorithmic}
\end{algorithm}
\subsection{Efficient Sampling from $\boldsymbol{\piout}$}
The structure of our efficient sampling scheme leverages the closed form solution of~\eqref{eq:outputDPO} which is offered by the following lemma.
\begin{lemma} \label{lemma:closed_form}
    The policy $\piout$ defined as solution of the optimization problem in \eqref{eq:outputDPO} satisfies for all $x,a\in \X\times\A$
    \[
\piout(a|x) = \frac{\min_{\ell\in[L]}\tilde{\pi}^{\ell}(a|x) e^{-\zeta^\ell(x)}}{\sum_{a\in\A}\min_{\ell\in[L]}\tilde{\pi}^{\ell}(a|x) e^{-\zeta^\ell(x)} }.
\]
\end{lemma}
For the following discussion, let us denote the numerator of $\piout$ as $$f_{\mathrm{out}}(x,a) := \min_{\ell\in[L]}\tilde{\pi}^{\ell}(a|x) e^{-\zeta^\ell(x)},$$ and recall that $a$ stands here for a full response not simply a token. Therefore, the computation of the denominator which involves a sum over the responses space $\A$ is not tractable. However, rejection sampling (see \Cref{alg:rejsampl}) is a very attractive solution in our scenario. Indeed, any model in the ensemble can serve as an ideal proposal distribution since it upper bounds $f_{\mathrm{out}}(x,a)$ everywhere. To see this, notice that $\min_{\ell\in[L]}\tilde{\pi}^{\ell}(a|x) \leq \tilde{\pi}^{\ell'}(a|x) $ for any $\ell'$. These facts ensure that for any model $\tilde{\pi}^\ell$ the ratio  $f_{\mathrm{out}}(x,a)/\tilde{\pi}^\ell(a|x) \leq 1$ for all $x,a\in \X\times\A$.
Formally, the following lemma shows that with high probability \Cref{alg:rejsampl} will output a valid sample from $\piout$ in polynomial time.
\begin{lemma}
Let us consider an arbitrary prompt $x\in\X$ and recall that $f_{\mathrm{out}}(x,a)$ satisfies that $\piout(a|x) = f_{\mathrm{out}}(x,a)/\sum_{a\in\A} f_{\mathrm{out}}(x,a)$. Then,
\label{lemma:rejsampling}
\Cref{alg:rejsampl}, with access to a  proposal distribution $\pi_{\mathrm{prop}}$ such that $f_{\mathrm{out}}(x,a)/\pi_{\mathrm{prop}}(a|x) \leq 1$,   outputs a valid sample from $\piout(\cdot|x)$ after sampling at most $N_{\RS}(x) = \frac{\log(1/\delta)}{\sum_{a\in\A}f_{\mathrm{out}}(x,a)} $ times from $\pi_{\mathrm{prop}}(\cdot|x)$  with probability $1-\delta$.
\end{lemma}

An interesting remark is that the number of trials $N_{\RS}(x)$
depends on the inverse of the \emph{ensemble disagreement} measured via $\sum_{a\in
A} f_{\mathrm{out}}(x,a)$. In words, the rejection sampling routine is expected to take longer to respond to prompts for which the ensemble members do not put a reasonably large probability on a common answer.   

\paragraph{Discussion of the hyperparameter} In our practical implementation  $\lambda^\ell$ and $L$ are treated as hyperparameters. In practice, we did not observe the practical performance to be significantly impacted by the choices for $\lambda^\ell$, at least in the ablation range we considered. $L$ is a more important hyperparameter as shown by our ablation presented in \Cref{app:exp}.

Finally, we provide our algorithm pseudocode in \Cref{alg:pepo} and we move to the theoretical analysis of \texttt{PEPO}.

\subsection{Theoretical Guarantees}
\label{sec:theory}
We present here the theoretical guarantees for $\texttt{PEPO}$ focusing on the tabular setting where the number of prompts and responses are finite. 
Our theoretical goal is to quantify the required dataset size $N$, so that with high probability $J_\beta(\pi^{\star}) - J_\beta(\piout) \leq \varepsilon$ for any policy $\pi^\star\in \Pi$. In words, we want the suboptimality of the policy learned by $\texttt{PEPO}$ w.r.t. the comparator $\pi^\star$ to be at most $\varepsilon$.
In particular, as motivated by \cite{huang2024correcting,liu2024provably} the gold standard for arguing that over-optimization is avoided is to prove a bound on $N$ scaling only with a \emph{single-policy concentrability} coefficient defined as $$C^\star := \sum_{x\in \X}\initial(x)\sum_{a\in\A}\frac{(\pi^\star(a|x))^2}{\pidata(a|x)}.$$  
In stark contrast, methods prone to over-optimization, such as DPO, incur in an \emph{all-policy concentrability} coefficient \cite{song2024importance} defined as
$C^{\mathrm{all}} := \max_{\pi\in \Pi}\sum_{x\in \X}\initial(x)\sum_{a\in\A}\frac{(\pi(a|x))^2}{\pidata(a|x)}.$
We highlight that $C^\star$ and $C^{\mathrm{all}}$ have a strongly differently qualitative behaviours. On the one hand, $C^\star$ is reasonably small as soon as sampling from $\pidata$ allows to observe with high enough probability all the actions which are often chosen by $\pi^\star$. On the other hand, $C^{\mathrm{all}}$ imposes basically to observe with positive probability all possible responses in the dataset \footnote{If $\Pi$ contains the uniform policy over responses.}, even those which would never be taken by $\pi^\star$. For our fast rate, we will also need a stroner notion of single policy concentrability coefficient, sometimes known as $L_{\infty}$ concentrability coefficient, defined as
$$
C^\star_\infty = \max_{x,a} \frac{\pi^\star(a|x)}{\pidata(a|x)}.
$$
Clearly, it holds that $C^\star \leq C^\star_\infty $.
We are now ready to present our main result: the sample complexity guarantees for \texttt{PEPO} reported in \Cref{thm:pepomain} which depends only on $C^\star$ and not on $C^\mathrm{all}$. 
\begin{theorem}
\label{thm:pepomain}
Let us consider running \texttt{PEPO} (\Cref{alg:pepo}) with $L=\ceil{\frac{7\log(\nicefrac{\abs{\X}\abs{\A}^2}{\delta})}{2}}$ many ensembles and a preferences dataset $\mathcal{D}$ of size $N$. Then, with probability at least $1-\delta$, we have that the suboptimality of the \texttt{PEPO} output, $\piout$, against any comparator policy $\pi^\star : \X\rightarrow\Delta_{\A}$ and
for any value of $\beta\in (0,\infty)$ is upper bounded as follows,
\[
J_\beta(\pi^{\star}) - J_\beta(\piout) \leq \widetilde{\mathcal{O}}\br{\min \bcc{\sqrt{\frac{C^\star e^{\RMAX} \abs{\X} \abs{\A}^2 \log \delta^{-1}}{N}}, \frac{C_\infty^\star e^{\RMAX} \abs{\X} \abs{\A}^2 \log \delta^{-1}}{\beta N} }}.
\]
\end{theorem}
In \Cref{thm:pepomain}, the notation $\widetilde{\mathcal{O}}(\cdot)$ hides logarithmic factors in $\abs{\X}$, $\abs{\A}$, $N$ and polynomial in $\RMAX$. The exponential dependence on $\RMAX$ is conjectured to be unavoidable in the offline setting \cite{das2024active}\footnote{\cite{chen2025avoiding} moved the $e^{\RMAX}$ dependence in lower order terms but in the online setting.}. 
An important consequence of the bound is that if the dataset size is  $\widetilde{\mathcal{O}}\br{\frac{C^\star e^{\RMAX} \abs{\X} \abs{\A}^2 \log \delta^{-1}}{\varepsilon^2}}$, we can provably attain our learning goal, i.e., that $J_\beta(\pi^{\star}) - J_\beta(\piout) \leq  \varepsilon$.

Moreover, in the high regularization regime, i.e., for high values of $\beta$,
and in setting in which $C^\star_\infty$ is comparable to $C^\star$,
we have that the sample complexity improves to $\mathcal{O}\br{\frac{C^\star_{\infty} e^{\RMAX} \abs{\X} \abs{\A}^2 \log \delta^{-1}}{\beta \varepsilon}}$.

Critically, our bound does not require prior knowledge of $C^\star$ while this is needed for obtaining comparable rates in \cite{huang2024correcting,liu2024provably} which works however in the function approximation rather than tabular setting.\footnote{Without knowledge of $C^\star$, \cite{huang2024correcting,liu2024provably} present rates of order $\sqrt{(C^\star)^2/N}$. }

For our analysis, we consider that the dataset partitions $\cup_{\ell\in[L]}\mathcal{D}^\ell = \mathcal{D}$ is such that  $ N(x,a,b)/L - 1 \leq N^\ell(x,a,b) \leq N(x,a,b)/L $ where $N(x,a,b)$ is the number of times $x,a,b$ is observed in $\mathcal{D}$ and $N^\ell(x,a,b)$ is the number of observations in $\mathcal{D}^\ell$. That is, if data points are repeated, they are split evenly across the partition.
We now briefly present the proof sketch of \Cref{thm:pepomain}.
The analysis has three main conceptual steps.
\paragraph{Step 1: Establishing pessimism} A fundamental quantity of the analysis is the implicitly estimated reward gap , defined as \begin{align*} \widehat{\Delta}(x,a,b) &:= \beta \min_{\ell\in[L]} \br{\log \frac{\tilde{\pi}^\ell(a|x)}{\piref(a|x)} - N^{-\ell}(x)\sum_{X,A \in \cD\setminus\cD^\ell} \log \frac{\tilde{\pi}^\ell(A|X)}{\piref(A|X)}\mathds{1}\bcc{X = x}} \\&~~~ - \beta\max_{\ell\in[L]} \br{\log \frac{\tilde{\pi}^\ell(b|x)}{\piref(b|x)} - N^{-\ell}(x)\sum_{X,A \in \cD\setminus\cD^\ell} \log \frac{\tilde{\pi}^\ell(A|X)}{\piref(A|X)}\mathds{1}\bcc{X = x}}\end{align*}
which with high probability is a lower bound on the average reward gap under $r^\star$  when the second action is sampled from $\pidata$.
The result is formalized by the following lemma.
\begin{restatable}{Lem}{pessmain} \label{lemma:pessmain}
Let us denote $\sum_{b\in \A} \pidata(b|x) \Delta(x,a,b)$ as $\Delta(x,a,\pidata)$ for $\Delta \in \{\Delta_{r^\star}, \widehat{\Delta}\}$ and let us assume that $L \geq \frac{7 \log(\nicefrac{\abs{\X}\abs{\A}^2}{\delta})}{2}$. Then, the policies $\bcc{\tilde{\pi}^\ell}^L_{\ell=1}$, solutions of \eqref{eq:RegDPO} with $$\lambda^\ell(x,a,b) = \sigma^{-1}\br{\frac{N^\ell(x,a \succ b)}{N^\ell(x,a, b)}} -  \sigma^{-1}\br{\frac{N^\ell(x,a \succ  b)}{N^\ell(x,a, b)+2}} , ~~\text{if}~~N^\ell(x,a,b) > 0,$$ and $\lambda^\ell(x,a,b) = 0$ otherwise, induce an estimated gap $\widehat{\Delta}$ such that with probability at least $1-\delta$, 
    $$ \widehat{\Delta}(x,a,\pidata) \leq \Delta_{r^\star}(x,a,\pidata)~~\forall x,a\in \X\times\A.$$
\end{restatable}
\paragraph{Step 2: From suboptimality to estimation error}
The second step of the analysis consists of upper bounding the suboptimality of $\piout$ in terms of how accurately $\widehat{\Delta}$ estimates $\Delta_{r^{\star}}$ in average under $\pidata$.
The pessimistic event that holds with high probability in light of Lemma~\ref{lemma:pessmain} ensures that in this conversion only $C^\star$ appears and not $C^{\mathrm{all}}$. 
The formal result reads as follows.
\begin{restatable}{Lem}{conversion} 
\label{lemma:conversion}
Let us assume that the policies ensemble $\bcc{\tilde{\pi}^\ell}^L_{\ell=1}$ induces an estimated gap $\widehat{\Delta}$ such that $\widehat{\Delta}(x,a,\pidata) \leq \Delta_{r^\star}(x,a,\pidata)$ for all $x,a\in\X\times\A$. Then, for any $\beta > 0$,
\[
J_\beta(\pi^\star) - J_\beta(\piout) \leq \min \bcc{\sqrt{C^\star \innerprod{\pidata}{\mathrm{Err}^2}}, \beta^{-1} C^\star_\infty \innerprod{\pidata}{\mathrm{Err}^2}}
\]
where $\piout$ is computed from  $\bcc{\tilde{\pi}^\ell}^L_{\ell=1}$ as in \Cref{lemma:closed_form} and, $$\mathrm{Err}(x,a) :=  \widehat{\Delta}(x,a,\pidata) - \Delta_{r^\star}(x,a,\pidata).$$
\end{restatable}

\paragraph{Step 3: Bounding the estimation error}
Given Lemma \ref{lemma:conversion}, the last step of the proof is to provide a high probability bound on the estimation error $\innerprod{\pidata}{\mathrm{Err}^2}$. The following result gives precisely this result and leverages the intuition that \eqref{eq:RegDPO} can be seen as a maximum likelihood problem for estimating $\Delta_{r^\star}$.
\begin{restatable}{Lem}{concentrationlem}\label{lemma:concentration_main}
It holds with probability at least $1-5\delta$ that
\[
\innerprod{\pidata}{\mathrm{Err}^2} \leq \widetilde{\mathcal{O}}\br{\frac{ e^{\RMAX} \abs{\X}\abs{\A}^2 L^2\log(\delta^{-1})}{N}}.
\]
\end{restatable}
At this point, the proof of \Cref{thm:pepomain} clearly follows by combining Lemmas \ref{lemma:pessmain},\ref{lemma:conversion} and \ref{lemma:concentration_main}. 
The detailed analysis is presented in \Cref{app:analysis}.

\paragraph{On the optimal ensemble size}
Our analysis provides precise guidelines for choosing $L$. In particular, Lemma \ref{lemma:concentration_main} shows that the generalization error increases with $L$. Therefore, we would like to choose it as small as possible. However, Lemma \ref{lemma:pessmain} imposes a lower bound on $L$ for the pessimism to hold. Hence, the conclusion that $L$ should be the smallest value which ensures that the estimator $\widehat{\Delta}$ is pessimistic with high enough probability.

\subsection{Detailed comparison with  $\chi^2$\texttt{PO} and \texttt{DPO} + \texttt{SFT} in the tabular setting.}
\label{sec:priorwork}
Looking at the bound in \cite{huang2024correcting} Theorem 3.1, we notice that $\chi^2 PO$ is upper bounded by $\mathcal{O}(e^{2R_{\max}}\sqrt{\frac{C^\star \log (|\Pi|/\delta)}{N}})$. When we are in the tabular case, we need to consider $|\Pi|$ to be a $1/N$-covering set of all Markov policies, i.e. all mappings $\mathcal{X} \rightarrow \Delta_{\mathcal{A}}$. The size of this set can be computed to be $N^{|\mathcal{X}||\mathcal{A}|}$. Therefore plugging into the $\chi^2 PO$ bound we obtain $\mathcal{O}(e^{2R_{\max}}\sqrt{\frac{C^\star |\mathcal{X}||\mathcal{A}| \log (N/\delta)}{N}})$ which also includes polynomial dependence on the number of states and action as our bound does. The $\chi^2 PO$ bound has a better action dependence, but it requires knowledge of $C^\star$. When unavailable, the $\chi^2PO$ bound becomes $$\mathcal{O}\br{e^{2R_{\max}}\sqrt{\frac{(C^\star)^2 |\mathcal{X}||\mathcal{A}| \log (N/\delta)}{N}}},$$ which would lead to a sample complexity of order $$\widetilde{\mathcal{O}} \left(e^{2R_{\max}} \frac{(C^\star)^2 |\mathcal{X}||\mathcal{A}| \log (1/\delta)}{\epsilon^2}\right)$$ which has a worse dependence on $C^\star$ compared to our analysis.
Moreover, the $\chi^2 PO$ bound requires knowledge of $\pi_{\mathrm{data}}$.

Looking at the bound in \cite{liu2024provably} Theorem 5.3, we have that in the tabular setting the covering number for the reward class scales as $N^{|\mathcal{X}||\mathcal{A}|}$. Therefore, in the tabular setting $\iota = \mathcal{O}(\sqrt{|\mathcal{X}||\mathcal{A}| \log (N/\delta)})$. Plugging this into the main bound of (\cite{liu2024provably} Theorem 5.3), assuming that $\pi_{\mathrm{ref}}$ (or $\pi_{\mathrm{base}}$ in their notation) equals $\pi_{\mathrm{data}}$  and neglecting the lower order terms we obtain an upper bound of order $\frac{(C^\star)^2 e^{2R_{\max} \sqrt{|\mathcal{X}||\mathcal{A}| \log (N/\delta)}} }{\sqrt{N}}$ which leads to a sample complexity of order $$\widetilde{\mathcal{O}} \left(e^{4R_{\max}} \frac{(C^\star)^4 |\mathcal{X}||\mathcal{A}| \log (1/\delta)}{\epsilon^2}\right).$$
Therefore, again the action dependence is better but the dependence on $C^\star$ is worse than the one attained by \texttt{PEPO}, even without knowledge of $\pi_{\mathrm{data}}$. Moreover, our analysis achieves a fast rate in the large $\beta$ regime when $C^\star$ is comparable to $C^\star_\infty$. It is unclear if $\chi^2$PO and \texttt{DPO+SFT} can attain the same.
\section{Experiments}
\label{sec:experiments}
In this section, we show that \texttt{PEPO} avoids over-optimization in LLMs post-training experiments.

\paragraph{\texttt{PEPO} avoids over-optimization}
Our experiments  in \Cref{fig:winrate-initial} clearly  show that DPO over-optimizes: the win rate increases during the first few epochs but then drops off  significantly. This effect is evident for all base models we experimented with.  Moreover, we point out in \Cref{app:earlystopfails} that early stopping DPO at best evaluation loss leads to the selection of a suboptimal checkpoint. Therefore, early stopping can not be considered an effective way of circumventing this negative DPO behaviour.  In contrast, token-level \texttt{PEPO} win rates increase and then sstabilize during training once $L$ is set to a large enough values ($L=3$ in Llama and $L=2$ for the 7B models.). 
These results show that avoiding the $C^{\mathrm{all}}$ dependence in the theoretical analysis translates into avoiding over-optimization in practice. In Zephyr and Mistral, we also observe that increasing $L$ further leads to deteriorated performance because each ensemble members leverage a smaller dataset. The deterioration with an excessively large value of $L$ is consistent with our Lemma \ref{lemma:concentration_main}.  

\begin{figure*}[t]
    \centering
\includegraphics[width=\textwidth]{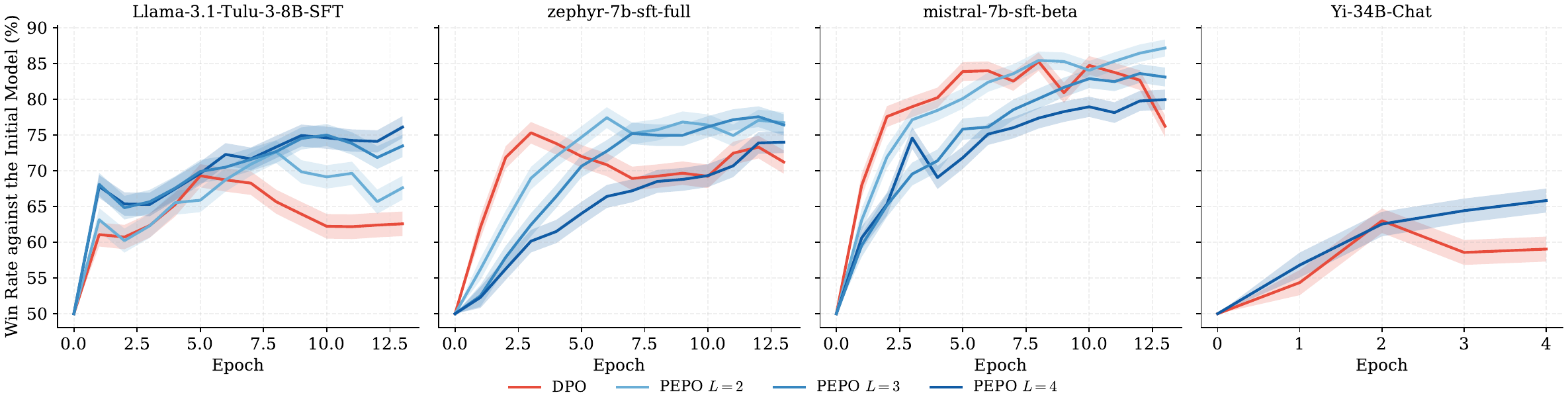}
    \caption{Win rate (\%) against the initial model on AlpacaEval2 across training epochs. Shaded regions indicate standard error of the mean across per-instruction preferences. We compare DPO and PEPO with varying ensemble sizes $L$. 
    }
    \label{fig:winrate-initial}
\end{figure*}

\paragraph{Comparison to $\chi^2$PO and SFT+DPO}
We compare \texttt{PEPO} against two methods specifically designed to mitigate overoptimization: $\chi^2$PO and SFT+DPO. As shown in \Cref{tab:best-winrates}, \texttt{PEPO} consistently outperforms both baselines across the models we evaluate. Crucially, our experimental setup uses the UltraFeedback dataset, where the responses are generated by a diverse pool of LLMs (including proprietary models like GPT-4), meaning that $\pidata \neq \piref$.  In this setting, $\chi^2$PO and SFT+DPO lose their theoretical guarantees, as both methods rely on exact or approximate knowledge of $\pidata$.  \texttt{PEPO}, by contrast, uses a valid pessimism mechanism even when $\pidata$ is unknown, hence achieves higher winrates. The performance \begin{wrapfigure}[19]{r}{0.5\linewidth}
    \begin{tcolorbox}[
        colback=blue!06!white,
        colframe=blue!80!black,
        width=\linewidth,
        before skip=0pt, 
        after skip=0pt,
        left=2pt, right=2pt, 
        top=0.5pt, bottom=0.5pt
    ]
    \centering
    \begin{tikzpicture}
    \begin{scope}[xshift=0cm]
        \begin{axis}[
            scale=0.4,
            xlabel={generation time (ms/token)},
            ylabel={Win Rate (\%)},
            grid=major,
            xmin=0, xmax=80,
            ymin=55, ymax=78,
            legend style={at={(0.98,0.02)}, anchor=south east, font=\tiny},
            every axis title/.style={font=\normalsize\bfseries},
        ]
            \addplot[only marks, mark=triangle*, mark size=2pt, color=gray, thick] coordinates {(6.571, 61.00)};
            \addplot[only marks, mark=square*, mark size=2pt, color=teal, thick] coordinates {(23.382, 67.80)};
            \addplot[only marks, mark=*, mark size=2pt, color=red!70!black, thick] coordinates {(46.422, 72.66)};
            \legend{DPO, \texttt{PEPO} token, \texttt{PEPO} rejection}
        \end{axis}
    \end{scope}
\end{tikzpicture}
    \caption{Comparison of win rate vs generation time in Llama after $1$ epoch. \texttt{PEPO} token-level attains the best tradeoff. \label{tab:refsamp_vs_token}}
    \resizebox{\linewidth}{!}{
            \begin{tabular}{lccc}
                \toprule
                & Token-level \texttt{PEPO} & $\chi^2$PO &  SFT+DPO \\
                \midrule
                Llama-3-8B & \textbf{76.1 $\pm$ 1.5} & 70.7 $\pm$ 1.6 &  69.7 $\pm$ 1.6 \\
                Mistral-7B & \textbf{87.2 $\pm$ 1.2} & 82.8 $\pm$ 1.3 &  72.4 $\pm$ 1.6 \\
                \bottomrule
            \end{tabular}}
\caption{Win rates (\%) at the respective best epochs against the initial model on AlpacaEval2. \label{tab:best-winrates}}
\end{tcolorbox}
\end{wrapfigure} gap narrows when the response dataset is generated directly from $\piref$ (i.e., $\pidata = \piref$), as this restores the conditions under which the guarantees of $\chi^2$PO and SFT+DPO hold. To verify this, we set up in \Cref{app:toy} a synthetic bandit experiment. 
\paragraph{Rejection sampling vs token-level \texttt{PEPO}}

In \Cref{tab:refsamp_vs_token}, we show the winrate generation time tradeoff for DPO and our two pessimistic aggregation rules: the rejection sampling routine in \Cref{alg:rejsampl} and the token-level implementation of \texttt{PEPO}. The results in \Cref{tab:refsamp_vs_token} show that rejection sampling leads to the best winrate, DPO to the fastest generation time, while token-level \texttt{PEPO} achieves the best tradeoff. Thus, we use the latter for our large-scale experiments above.
Finally we point out that, DPO is faster at generation times , but slower in training compared to PEPO by a factor $L$. That is because each ensemble member $\tilde{\pi}^\ell$ is trained in parallel with a dataset size which is equal to a $1/L$ fraction of the total data. 
\section{Related Work}

Some heuristic techniques have been proposed to mitigate over-optimization: imposing a trust region \citep{gorbatovski2024learn}, using length normalization \citep{meng2024simpo}, or using a loss function with a bounded minimizer, or better curvature \citep{azar2024general,agarwal2025design}. These methods can sometimes be effective in practice, but unfortunately, none of these approaches can provably fight back the risk of over-optimization. 

The literature is also rich of works affected by the opposite issue, indeed pessimistic RLHF approaches such as \citep{zhu2023principled,zhan2023provable,schlaginhaufen2025efficient,li2023reinforcement} enjoy strong bounds but are rarely applicable. The common idea is first to learn a reward function and its uncertainty from the preference dataset and, then, to optimize the reward minus the uncertainty. This procedure is provably robust to over-optimization but there are two main downsides: ($i$) relying on a two-step procedure and ($ii$) the fact that the uncertainty estimate is usually computed leveraging elliptical bonuses \citep{Abbasi-Yadkori:2011}, which is a technique that does not offer guidance towards a practical approach for modern LLMs\footnote{\citep{tuyls2025representation} is a notable exception which, however, uses an elliptical bonuses inspired technique for optimism in online algorithms and not for pessimistic purposes in the offline setting. }.

Moreover, there are previous works using ensembles in alignment \citep{coste2023reward,fisch2024robust} but they also rely on a two-step recipe. 
Both works \citep{coste2023reward,fisch2024robust} empirically verified that the technique helps fight over-optimization.

Existing theoretically grounded DPO variants that achieve single policy concentrability guarantees require either that the responses in the dataset are generated according to the SFT checkpoint $\piref$ \citep{huang2024correcting} or according to a policy that covers well a known policy, denoted as $\pi_{\mathrm{base}}$ in \citep{liu2024provably}. 
 \citep{liu2024provably} suggested to use the distribution of the chosen responses as $\pi_{\mathrm{base}}$. Therefore, in this case, imposing their assumption implies knowledge of $\pidata$.
Other pessimistic analysis of DPO \citep{cen2024value} still incurs in a all policy concentrability term, see the factor $\kappa_{\mathcal{D}}$ in \citep{cen2024value}[Theorem 2]. 
Another way to attain single policy concentrability bounds is to allow regeneration according to the policy sequence generated during learning \citep{ji2024self}, this process makes the implementation more costly compared to the original DPO algorithm.

Finally, we remark that the idea of using ensembles for uncertain aware decision making has been previously used in reinforcement and imitation learning, both at a theory-only \citep{osband2014near,osband2016generalization,o2021variational,o2023efficient} or also at a practical level \citep{osband2016deep, chen2017ucb, osband2018randomized,osband2023approximate,osband2023epistemic,tarbouriech2024probabilistic,ishfaq2021randomized,ishfaq2023provable,ishfaq2024more,ishfaq2025langevin}. Finally, at a technical level, the most inspiring works for our analysis are \citep{cassel2025batch}, which studied ensemble-based exploration in bandits and \citep{viel2025soar}, which studies Imitation Learning (IL) with access to the environment. In particular, \citep{viel2025soar} used an ensemble-based exploration scheme to make previous theoretical insights on the importance of exploration in IL from states or reward features-only \citep{Shani:2021,viano2024imitation,moulin2025optimistically} applicable to scenarios where neural network function approximation is required. 

\section{Conclusions}
\label{sec:conclu}

In this paper, we introduced \texttt{PEPO}, the first DPO-like algorithm that provably and practically avoids over-optimization without knowing the data distribution of the responses recorded in the dataset. This situation is very common in practice when post-training models by leveraging responses that are either human or proprietary models generated.  This contribution opens several exciting directions, the most important of all being to establish the same result beyond the tabular setup. Indeed, while the good experimental results suggest that the tabular setting might be an accurate enough model for practical post-training, there are theoretical results \citep{huang2024correcting,liu2024provably} which go beyond it but at the cost of knowing $\pidata$. We consider getting the best of both worlds an extremely exciting open problem. 

Another interesting direction inspired by \citep{yadkori2024believe} is to investigate if the notion of ensemble disagreement can be used to detect the risk of hallucinations and intervene, forcing the model to abstain from answering. In particular, we could force the model to respond ``\emph{I do not know}.'' if \cref{alg:rejsampl} does not terminate within a maximum number of attempts.

Several interesting open questions are left open by our work. We highlight the most exciting ones in the following.

\paragraph{Pessimism in general function classes with unknown $\boldsymbol{\pidata}$}
As already mentioned in the main text, we think that the most interesting open question is to develop a DPO-like method applicable to general function approximation, that achieves sample complexity bounds depending only on $C^\star$. We highlight this open question in \Cref{tab:main_open_q}.
In other words, \Cref{tab:main_open_q} shows that \texttt{PEPO}, DPO+SFT and $\chi^2$PO all applies beyond the easiest case of tabular setting with known $\pidata$ but \emph{along different axes}. Achieving those two improvements simultaneously, i.e. designing an approach which applies for general policy classes without knowledge about $\pidata$, is in our opinion the most important open question rising from our work. Moreover, we notice that the analysis in \cite{huang2024correcting,liu2024provably} still requires the policy class to be very expressive since it assumes that it realizes the solution of the objective in \eqref{eq:goal} over the domain of all tabular policies. Bypassing such an assumption is another interesting open direction.
\begin{table}[ht]
    \centering
    \caption{Comparison of offline DPO-like algorithms with sample complexity scaling only with $C^\star$. \label{tab:main_open_q} }
    \renewcommand{\arraystretch}{1.5} 
    \begin{tabular}{l|cc}
        \toprule
        & \textbf{Tabular Setting} & \textbf{General Policy Class} \\
        \midrule
        \textbf{Known $\pi_{\text{data}}$} & 
        \makecell[l]{\texttt{PEPO} (Ours) \\ DPO+SFT \cite{liu2024provably} \\ $\chi^2$PO \cite{huang2024correcting}} & 
        \makecell[l]{\cite{liu2024provably} (Convex only) \\ $\chi^2$PO \cite{huang2024correcting}} \\
        \midrule
        \textbf{Unknown $\pi_{\text{data}}$} & 
        \texttt{PEPO} (Ours) & 
        \textbf{\textit{Open Question}} \\
        \bottomrule
    \end{tabular}
\end{table}
\paragraph{Optimism via Ensemble} In the context of LLM post-training there are applications in which at training time it is possible to generate sentences from the model which is currently learning and getting feedback for this response under the form of a reward signal or a preference bit. Examples of the first kind are the emerging family of algorithms for reasoning such as GRPO \cite{shao2024deepseekmath,deepseekai2025deepseekr1}, GRPO Done Right \cite{liu2025understanding} and TBRM \cite{yuan2025trajectory} while for online algorithms learning from preference signals we highlight Online DPO \cite{guo2024direct} and Nash Learning from Human Feedback \cite{munos2024nash}.
While promising at an empirical level, they all miss an exploration mechanism that should be included to ensure that the learning model queries the reward or preference oracle for response which are uncertain, i.e. for which new information should really be acquired. Such lack of exploration has been also observed in certain experiments \cite{bamba2025xrpo,ding2025multi}. 
While recently some exploration mechanism have been proposed for improving GRPO, Best of N and Online DPO \cite{xie2024exploratory,tuyls2025representation} none of this works leverages the ensemble mechanism which we show to be effective for pessimism. Therefore, it remains open to investigate optimistic ensemble based versions of the above algorithms and see if they can lead to any theoretical and practical advantage over the existing exploration methods.

\paragraph{Outputting a parametric policy}
Our work successfully avoids over-optimization without knowledge of $\pidata$ at the cost of outputting a non-parametric policy from which we can sample efficiently but that we can not compute in closed form. Therefore, our method is ideal for being the last step of an LLM training pipeline (which is expected to be most often the case for DPO-like algorithms). However, \texttt{PEPO} is not suitable as an intermediate step of the training pipeline because it does not output a parametric policy, i.e. a training checkpoint in a more applied jargon. It is therefore interesting for future work to solve the exact same question this work solves but using parametric policies.

\section*{Acknowledgements}

Luca Viano is thankful to Adrian Müller for his insightful feedback about this work.  
This work was supported by Swiss Data Science Center under fellowship number P22\_03. Research was sponsored by the Army Research Office and was accomplished under Grant Number W911NF-24-1-0048. This work was funded  by the Swiss National Science Foundation (SNSF) under grant number 2000-1-240094. 
 This work was supported under project ID $\#$ 37 as part of the Swiss AI Initiative, through a grant from the ETH Domain and computational resources provided by the Swiss National Supercomputing Centre (CSCS) under the Alps infrastructure.

\newpage
\bibliographystyle{plainnat}
\bibliography{references}
\newpage
\appendix
\section{Additional Experiments}
\label{app:exp}
In this section, we report some experiments omitted from main text and provide additional details about our setup. 
\subsection{Ablation for the ensemble size $L$ in a controlled setting.}
We run an ablation for the ensemble size $L$ in a simple and controlled setting. The environment under consideration, we consider an empty prompt set and $20$ possible answers. The preferences are generated via the Bradley-Terry model with rewards which are all gaussian distributed. In particular, all rewards distribution but the one of the optimal arm have mean equal to $1$ and variance equal to $1$. On the other hand, the optimal arm's distribution has mean equal to $2.9$ and variance equal to $1/2$. We generate preference datasets of different sizes $\mathcal{D}$ and run \texttt{PEPO} for different ensemble sizes $L$. We report the results in \Cref{fig:Lablation} where we see that for all the values we tried PEPO finds the optimal action for all the $5$ considered seeds. However, the strongest performance is obtained for $L=3$. Increasing $L$, above $3$ slows down learning. This is consistent with the theoretical guideline of setting $L$ to the smallest value which is large enough to ensure pessimism.
\begin{figure}[h]
    \centering
    \includegraphics[width=0.45\textwidth]{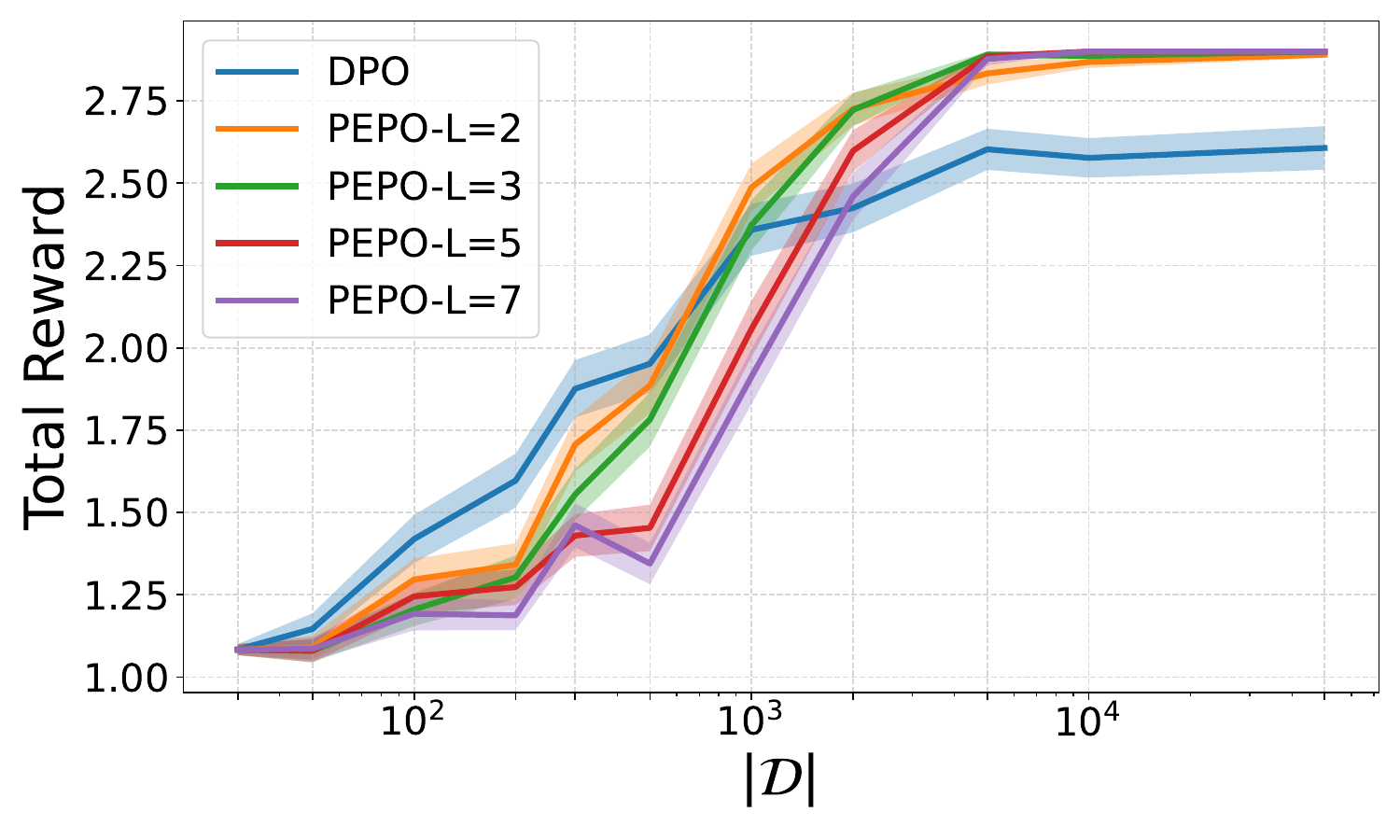}
    \caption{Ablation for $L$ in a bandit setting.}
    \label{fig:Lablation}
\end{figure}
\subsection{Known vs unknown $\boldsymbol{\pidata}$ in a controlled setting.}
\label{app:toy} 
We consider an empty prompt and $3$ possible actions. The preferences dataset is generated via the Bradley-Terry model with Gaussian rewards with means $[1,1.5,1]$ and variances $[1.5, 0.5, 1.5]$, i.e., the second action is the optimal one and also has lower variance. The other two actions are identical and suboptimal. 
\begin{figure}[t]
\centering
\begin{tabular}{cc} 
\subfloat[Known $\pidata$ \label{fig:known}]{
        \includegraphics[width=0.45\linewidth]{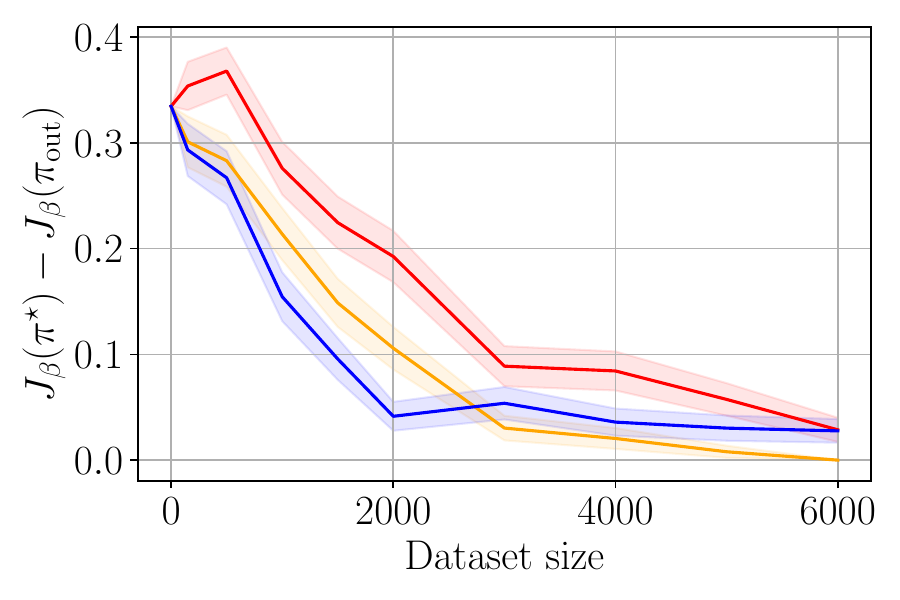}
    }\hfill
    \subfloat[\label{fig:unknown} Unknown $\pidata$]{
        \includegraphics[width=0.45\linewidth]{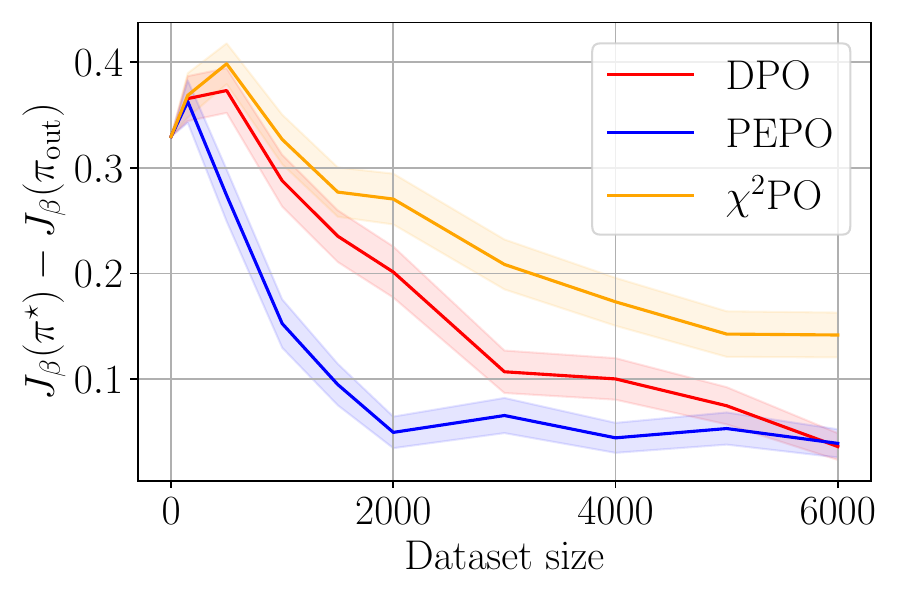}
    }
\end{tabular}
\caption{Experiment in $3$ arms bandit setting for the situation of known and unknown $\pidata$.}
\vskip -10pt
\end{figure}
In our experiment in \Cref{fig:known}, we assume knowledge of $\pidata$ and we run all the algorithms with $\piref=\pidata$. Consistently with the theoretical results, both $\chi^2$PO and \texttt{PEPO} outperforms DPO in this setting. 

However, when considering an unknown $\pidata$ ( see \Cref{fig:unknown} ), we can not set $\piref = \pidata$. 
For this experiment, we set $\piref$ to put most of the mass on a suboptimal action while $\pidata$ puts the largest weight on the optimal one. In this situation, $\texttt{PEPO}$ still works reliably while $\chi^2$PO fails because it interprets the high mean estimate for the optimal arm as affected by high uncertainty, as the optimal action has a small probability under $\piref$.
For \Cref{fig:known}, we used $\pidata=\piref=[ 0.04, 0.93, 0.03]$. That is, the largest mass is on the central action, which is the optimal one. This makes the optimal action well-covered, i.e. small $C^\star$, therefore $\chi^2$PO and \texttt{PEPO} can work reliably.  At the same time, the all-policy concentrability is large because the two suboptimal actions are poorly covered. This explains why DPO is outperformed in \Cref{fig:known}.
For the unknown case, we consider instead $\pidata = [ 0.04, 0.93, 0.03]$ and $\piref=[0.00025, 0.00025, 0.9995]$ and $\beta=0.001$. In this case, we have two distinct failure modes for DPO and $\chi^2$PO while \texttt{PEPO} still works reliably. DPO fails because the all-policy concentrability is still very large. This implies that the reward of the suboptimal actions can not be estimated accurately, and, therefore, there is the risk of playing a suboptimal action whose reward is wrongly overestimated. The failure mode of $\chi^2$PO is instead different. The mismatch between $\pidata$ and $\piref$ leads $\chi^2$PO to perceive the suboptimal action with high probability under $\piref$ as accurately estimated. On the contrary, since the optimal action has low mass under $\piref$, this action is perceived as quite uncertain. All in all, $\chi^2$PO concludes that the high reward estimate for the optimal action can not be trusted and plays a suboptimal action. \texttt{PEPO} still succeeds in this setting because under $\pidata$, we have that the optimal action is actually estimated more accurately than the suboptimal ones. \texttt{PEPO} can use this fact to successfully identify and play the optimal action if the dataset size is large enough. This is shown in \Cref{fig:unknown}.

To develop intuition and see this mechanism in an even more evident manner, we consider the same task with $\beta=0$. In this case, we can not derive DPO-like algorithms but we compute the maximum likelihood estimator $\hat{r}$ for the reward function from the preference dataset and then consider three approaches: (i) RL outputs the action with highest estimated reward, (ii) a method that we name $\chi^2$RL which outputs $\piout = \argmax_{\pi \in \Pi} \innerprod{\pi}{\hat{r}} - \gamma D_{\chi^2}(\pi, \piref)$ where we set $\gamma=40$ and (iii) a \texttt{PEPO}-like approach, dubbed PERL that instead of a single reward estimate computes $L$ of them, i.e. $\bcc{\hat{r}_\ell}^L_{\ell=1}$  and then outputs the action with highest \emph{robust reward} defined as $\min_{\ell\in[L]}\hat{r}_\ell$. We report the results of this experiment in \Cref{fig:betazero}.

\begin{figure*}[t] 
\centering
\begin{tabular}{cc} 
\subfloat[Known $\pidata$]{
        \includegraphics[width=0.45\textwidth]{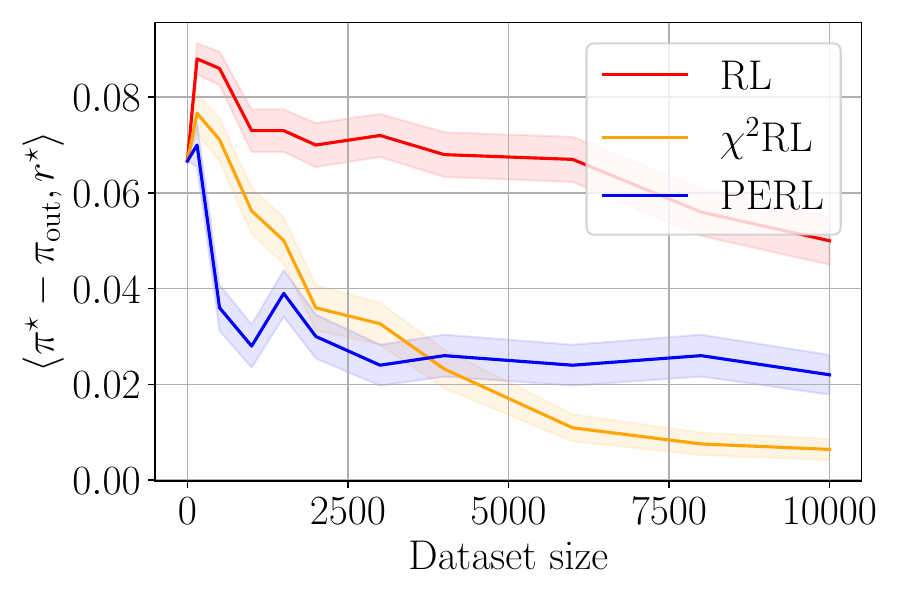}
    }
    &
    \subfloat[Unknown $\pidata$]{
        \includegraphics[width=0.45\textwidth]{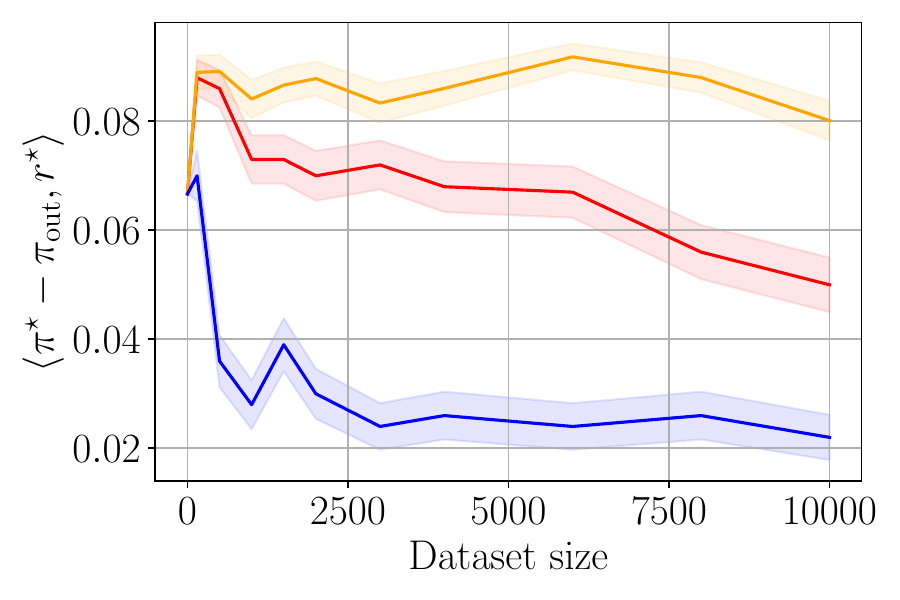}
    }
\\
\\
\end{tabular}
\caption{Result in the controlled setting without regularization, i.e. with $\beta=0$.}
\label{fig:betazero}
\end{figure*}
\subsection{Other worst-case aggregation rules for the ensemble. }
The rejection sampling rule with the minimum in the numerator as defined in \Cref{alg:rejsampl} can be over conservative in practice and leads to a very high rejection rate. One possible alternative to reduce the generation time of the model is to consider this alternative numerator,
$f_{\mathrm{out}}(x,a) = \bs{\frac{\sum^L_{\ell=1}\tilde{\pi}^{\ell}(a|x)}{L}
    - \eta \cdot \mathrm{std}(\bcc{\tilde{\pi}^{\ell}(a|x)}^L_{\ell=1})}\exp\br{-\zeta^\ell(x)}$
where $\eta$ is a hyperparameter that can be used to tune the level of pessimism. With the same technique used in \cite{viel2025soar}, we can still establish the sample complexity guarantees for this variant if we set $\eta = \sqrt{L-1}$. However, in practice, we observe a more permissive value of $\eta$, such as $0.1$, working better. We use this variant for the comparison with the token-level \texttt{PEPO} in \Cref{tab:refsamp_vs_token}.

\subsection{Details for the large scale experiments}
\label{app:large_experiments_details}
Here, we report details on our LLM experimental setup.
\subsubsection{Early-stopping in DPO selects a suboptimal checkpoint.}
\label{app:earlystopfails}
We point out that early-stopping is not an option because at training time it is impossible to monitor the generation quality, i.e. the win rate. One could monitor the DPO evaluation loss but this is a poor proxy for the win rate. In \Cref{fig:early_stop} we observe that higher values of evaluation loss correspond to better win rates and that in general evaluation loss-based early stopping leads to select a suboptimal checkpoint. Therefore, over-optimization can definitely not be bypassed by evaluation loss-based early stopping.

\begin{figure}[ht]
    \centering
    \large
    \begin{minipage}{0.3\textwidth}
        \centering
        \begin{tikzpicture}[scale=0.5]
            \begin{axis}[title={Llama-3.1-Tulu}, ylabel={\textcolor{blue}{Eval Loss}}, xlabel={Epoch}, grid=major]
                \addplot[blue, mark=*] coordinates {(0,0.6932) (1,0.4815) (2,0.4790) (3,0.5131) (4,0.6186) (5,0.7802) (6,0.8876) (7,0.9754) (8,1.0534) (9,1.1115) (10,1.1927)};
                \draw[blue, dashed, ultra thick] (axis cs:2,\pgfkeysvalueof{/pgfplots/ymin}) -- (axis cs:2,\pgfkeysvalueof{/pgfplots/ymax});
            \end{axis}
            \begin{axis}[axis y line*=right, axis x line=none, ymin=45, ymax=75]
                \addplot[red, mark=square*] coordinates {(0,50.00) (1,61.08) (2,60.73) (3,62.34) (4,65.25) (5,69.30) (6,68.73) (7,68.27) (8,65.68) (9,63.94) (10,62.23)};
                \draw[red, dotted, ultra thick] (axis cs:5,45) -- (axis cs:5,75);
            \end{axis}
        \end{tikzpicture}
    \end{minipage}
    \hfill
    \begin{minipage}{0.3\textwidth}
        \centering
        \begin{tikzpicture}[scale=0.5]
            \begin{axis}[title={zephyr-7b-sft}, xlabel={Epoch}, grid=major]
                \addplot[blue, mark=*] coordinates {(0,9.0628) (1,6.1304) (2,5.1855) (3,4.6557) (4,4.3897) (5,4.2137) (6,4.3103) (7,4.5515) (8,4.8146) (9,4.9365) (10,4.8717)};
                \draw[blue, dashed, ultra thick] (axis cs:5,\pgfkeysvalueof{/pgfplots/ymin}) -- (axis cs:5,\pgfkeysvalueof{/pgfplots/ymax});
            \end{axis}
            \begin{axis}[axis y line*=right, axis x line=none, ymin=45, ymax=80]
                \addplot[red, mark=square*] coordinates {(0,50.00) (1,62.08) (2,71.87) (3,75.32) (4,73.82) (5,72.01) (6,70.84) (7,68.93) (8,69.26) (9,69.67) (10,69.23)};
                \draw[red, dotted, ultra thick] (axis cs:3,45) -- (axis cs:3,80);
            \end{axis}
        \end{tikzpicture}
    \end{minipage}
    \hfill
    \begin{minipage}{0.3\textwidth}
        \centering
        \begin{tikzpicture}[scale=0.5]
            \begin{axis}[title={Yi-34B-Chat}, xlabel={Epoch}, grid=major]
                \addplot[blue, mark=*] coordinates {(0,0.6931) (1,0.4875) (2,0.5812) (3,0.7862) (4,0.9546)};
                \draw[blue, dashed, ultra thick] (axis cs:1,\pgfkeysvalueof{/pgfplots/ymin}) -- (axis cs:1,\pgfkeysvalueof{/pgfplots/ymax});
            \end{axis}
            \begin{axis}[axis y line*=right, axis x line=none, ymin=48, ymax=65, ylabel={Winrate}, ylabel style={color=red}]
                \addplot[red, mark=square*] coordinates {(0,50.00) (1,54.36) (2,63.01) (3,58.57) (4,59.05)};
                \draw[red, dotted, ultra thick] (axis cs:2,48) -- (axis cs:2,65);
            \end{axis}
        \end{tikzpicture}
    \end{minipage}
    \caption{DPO Training Curves. Blue dashed: min eval loss; Red dotted: max winrate. As noticeable, early stopping according to the evaluation loss leads to selecting a suboptimal checkpoint.\label{fig:early_stop}}
\end{figure}
\subsubsection{Models and Fine-Tuning Setup}

We evaluate \texttt{PEPO} on four instruction-tuned language models spanning a range of sizes. At the 7--8B scale, we consider Llama-3.1-Tulu-3-8B-SFT, which serves as a strong open-source baseline, along with Zephyr-7B-SFT and Mistral-7B-SFT. To demonstrate scalability, we also include Yi-34B-Chat, a 34B parameter model. All backbone models have undergone supervised fine-tuning (SFT) prior to preference optimization.

Rather than full fine-tuning, we employ Low-Rank Adaptation (LoRA)~\citep{hu2022lora} for parameter-efficient training. All LoRA adapters use rank $r=16$ with scaling factor $\alpha=16$ (the low-rank updates are scaled by $\alpha/r$), a dropout rate of 5\%, and are applied to all linear layers in the model. This adds approximately 0.5-1.5\% (depending on the model) parameters which are trainable while the backbone stays frozen.

\paragraph{Memory Efficiency of the Ensemble.} A crucial practical advantage of \texttt{PEPO} is that maintaining an ensemble of $L$ models does \emph{not} require $L$ times more GPU memory compared to a single model. Since all ensemble members share the same frozen backbone and only differ in their lightweight LoRA adapters, the total memory footprint is approximately $M_{\text{base}} + L \cdot M_{\text{adapter}}$ rather than $L \cdot M_{\text{base}}$. Given that each adapter constitutes only small fraction of the base model size, the memory overhead of the ensemble is modest even for moderate values of $L$.
The schematic implementation of this process is shown in \cref{fig:parallel_inference}.

\begin{figure}
\centering
\begin{tabular}{c} 
\subfloat[Parallel Training \label{fig:parallel_training}]{
    \begin{tikzpicture}[scale=0.55, transform shape]
        \draw[thin] (0, 0) -- (0, 4.2);
        \foreach \x in {5, 10} \draw[dotted, gray] (\x, 0) -- (\x, 4.2);

        \node[anchor=east] at (-0.5, 3.4) {\textbf{GPU 1}};
        \node[anchor=east] at (-0.5, 2.2) {\textbf{GPU 2}};
        \node[anchor=east] at (-0.5, 0.5) {\textbf{CPU}};
        
        \draw[dotted, gray] (-6.0, 2.8) -- (18.5, 2.8); 
        \draw[solid, thin] (-6.0, 1.4) -- (18.5, 1.4);  

        \node at (2.5, 4.3) {\textbf{Step 1}};
        \fill[blue!20] (0.2, 3.1) rectangle (3.2, 3.7); \fill[pink] (3.2, 3.1) rectangle (3.5, 3.7);
        \node at (1.7, 3.4) {Backbone}; \node[rotate=90, font=\tiny] at (3.35, 3.4) {1};
        \node[font=\scriptsize] at (1.7, 2.8) {Epoch 1};
        \fill[blue!20] (0.2, 1.9) rectangle (3.2, 2.5); \fill[pink] (3.2, 1.9) rectangle (3.5, 2.5);
        \node at (1.7, 2.2) {Backbone}; \node[rotate=90, font=\tiny] at (3.35, 2.2) {2};
        \node[font=\scriptsize] at (1.7, 1.6) {Epoch 1};
        \fill[blue!10] (0.7, 0.2) rectangle (3.7, 0.8); \fill[pink] (3.7, 0.2) rectangle (4.0, 0.8);
        \node[font=\tiny] at (2.2, 0.5) {M3 (Idle)};

        \node at (7.5, 4.3) {\textbf{Step 2}};
        \draw[->, dashed, blue!50!black, thick] (4.0, 0.8) .. controls (4.4, 1.2) .. (5.0, 3.1);
        \node[blue!50!black, font=\tiny, rotate=35] at (4.4, 1.8) {Acquire GPU};
        
        \fill[blue!20] (5.2, 3.1) rectangle (8.2, 3.7); \fill[pink] (8.2, 3.1) rectangle (8.5, 3.7);
        \node at (6.7, 3.4) {Backbone}; \node[rotate=90, font=\tiny] at (8.35, 3.4) {3};
        \node[font=\scriptsize] at (6.7, 2.8) {Epoch 1};
        \fill[blue!20] (5.2, 1.9) rectangle (8.2, 2.5); \fill[pink] (8.2, 1.9) rectangle (8.5, 2.5);
        \node at (6.7, 2.2) {Backbone}; \node[rotate=90, font=\tiny] at (8.35, 2.2) {1};
        \node[font=\scriptsize] at (6.7, 1.6) {Epoch 2};
        \fill[blue!10] (5.7, 0.2) rectangle (8.7, 0.8); \fill[pink] (8.7, 0.2) rectangle (9.0, 0.8);
        \node[font=\tiny] at (7.2, 0.5) {M2 (Idle)};

    \end{tikzpicture}} \\
    \subfloat[]{\begin{tikzpicture}[scale=0.55, transform shape]
        \draw[thin] (0, 0.5) -- (0, 4.3);
        \foreach \x in {5, 10, 15} \draw[dotted, gray] (\x, 0.5) -- (\x, 4.3);
        
        \draw[dotted, gray] (-6.0, 2.8) -- (18.5, 2.8);

        \def\u{0.2} \def\bw{0.8}

        \node[anchor=south] at (-5.1, 3.4) {\textbf{Input}};
        \draw[thick, gray] (-5.5, 2.2) rectangle (-5.5+\bw, 3.4);
        \foreach \i in {1,...,5} { \draw[gray, thin] (-5.5, 2.2 + \i*\u) -- (-5.5+\bw, 2.2 + \i*\u); }
        \foreach \j in {1,...,3} { \draw[gray, thin] (-5.5 + \j*\u, 2.2) -- (-5.5 + \j*\u, 3.4); }
        
        \draw[thick] (-4.7, 2.8) -- (-4.2, 2.8);
        \draw[thick] (-4.2, 1.8) -- (-4.2, 3.8);
        \draw[->, thick] (-4.2, 3.8) -- (-3.5, 3.8);
        \draw[->, thick] (-4.2, 1.8) -- (-3.5, 1.8);

        \draw[thick, gray] (-3.5, 3.5) rectangle (-3.5+\bw, 4.1);
        \foreach \i in {1,2} { \draw[gray, thin] (-3.5, 3.5+\i*\u) -- (-3.5+\bw, 3.5+\i*\u); }
        \foreach \j in {1,...,3} { \draw[gray, thin] (-3.5+\j*\u, 3.5) -- (-3.5+\j*\u, 4.1); }
        \draw[thick, gray] (-3.5, 1.5) rectangle (-3.5+\bw, 2.1);
        \foreach \i in {1,2} { \draw[gray, thin] (-3.5, 1.5+\i*\u) -- (-3.5+\bw, 1.5+\i*\u); }
        \foreach \j in {1,...,3} { \draw[gray, thin] (-3.5+\j*\u, 1.5) -- (-3.5+\j*\u, 2.1); }

        \node[anchor=east] at (-0.5, 3.8) {\textbf{GPU 1}};
        \node[anchor=east] at (-0.5, 1.8) {\textbf{GPU 2}};

        \newcommand{\drawInferUnit}[3]{
            \fill[blue!20] (#1+0.2, #2) rectangle (#1+3.2, #2+0.6);
            \node[font=\scriptsize] at (#1+1.7, #2+0.3) {Backbone};
            \foreach \i in {1,2,3} {
                \pgfmathsetmacro{\xoff}{3.2 + (\i-1)*0.3}
                \ifnum\i=#3
                    \fill[pink] (#1+\xoff, #2) rectangle (#1+\xoff+0.3, #2+0.6);
                \else
                    \draw[dotted, pink, thick] (#1+\xoff, #2) rectangle (#1+\xoff+0.3, #2+0.6);
                \fi
                \node[rotate=90, font=\tiny] at (#1+\xoff+0.15, #2+0.3) {\i};
            }
            \draw[->, black] (#1+1.7, #2) -- (#1+1.7, #2-0.3);
            \fill[gray!10] (#1+1.6, #2-0.9) rectangle (#1+1.8, #2-0.3); 
            \draw[gray!50] (#1+1.6, #2-0.9) rectangle (#1+1.8, #2-0.3);
            \foreach \i in {1,2} { \draw[gray!50, thin] (#1+1.6, #2-0.9+\i*\u) -- (#1+1.8, #2-0.9+\i*\u); }
        }

        \node at (2.5, 4.8) {\textbf{Step 1}}; \drawInferUnit{0}{3.5}{1} \drawInferUnit{0}{1.5}{1}
        \node at (7.5, 4.8) {\textbf{Step 2}}; \drawInferUnit{5}{3.5}{2} \drawInferUnit{5}{1.5}{2}
        \node at (12.5, 4.8) {\textbf{Step 3}}; \drawInferUnit{10}{3.5}{3} \drawInferUnit{10}{1.5}{3}

        \node at (16.5, 4.8) {\textbf{Agg \& Append}};
        \foreach \rowy/\targety in {2.6/3.5, 0.6/1.5} {
            \foreach \i in {0,1,2} {
                \fill[gray!10] (15.3 + \i*0.3, \rowy) rectangle (15.5 + \i*0.3, \rowy+0.6);
                \draw[gray!50] (15.3 + \i*0.3, \rowy) rectangle (15.5 + \i*0.3, \rowy+0.6);
                \foreach \j in {1,2} { \draw[gray!50, thin] (15.3+\i*0.3, \rowy+\j*\u) -- (15.5+\i*0.3, \rowy+\j*\u); }
            }
            \draw[->, thick] (16.3, \rowy+0.3) -- (17.1, \rowy+0.3);
            \fill[gray!70] (17.1, \rowy) rectangle (17.3, \rowy+0.6); \draw[black] (17.1, \rowy) rectangle (17.3, \rowy+0.6);
            \foreach \j in {1,2} { \draw[black, thin] (17.1, \rowy+\j*\u) -- (17.3, \rowy+\j*\u); }
            \draw[->, thick] (17.2, \rowy+0.6) -- (17.2, \targety);
            \draw[thick, gray] (16.4, \targety) rectangle (16.4+\bw, \targety+0.6); 
            \foreach \i in {1,2} { \draw[gray, thin] (16.4, \targety+\i*\u) -- (16.4+\bw, \targety+\i*\u); }
            \foreach \j in {1,...,3} { \draw[gray, thin] (16.4+\j*\u, \targety) -- (16.4+\j*\u, \targety+0.6); }
            \fill[gray!70] (16.4+0.8, \targety) rectangle (16.4+1.0, \targety+0.6); \draw[black] (16.4+0.8, \targety) rectangle (16.4+1.0, \targety+0.6);
            \foreach \i in {1,2} { \draw[black, thin] (16.4+0.8, \targety+\i*\u) -- (16.4+1.0, \targety+\i*\u); }
        }

        \draw[->, blue!30, thick, dashed, rounded corners=15pt] (17.4, 3.8) -- (18.2, 5.3) -- (-3.1, 5.3) -- (-3.1, 4.1);
        \node[blue!40, font=\normalfont, fill=white, inner sep=1pt] at (7.5, 5.3) {Next Iteration};

    \end{tikzpicture}}
\end{tabular}
\caption{Parallel Inference (Sequential Adapter Switching) \label{fig:parallel_inference}}
\end{figure}

\subsubsection{Dataset and Preprocessing}

For all large-scale experiments, we utilize the UltraFeedback Binarized dataset~\citep{cui2023ultrafeedback}. This version is derived from the original UltraFeedback corpus by "binarizing" multi-model outputs into discrete preference pairs. Specifically, for each prompt, four responses were generated by diverse language models and scored by GPT-4 across multiple dimensions (e.g., helpfulness and honesty). The binarized version identifies the highest-scoring response as the \textit{chosen} sample and selects one of the remaining lower-scoring responses as the \textit{rejected} sample. We use the standard \texttt{train\_sft} split for training and \texttt{test\_sft} for validation. However, we note that the validation loss is a poor proxy for downstream evaluation performance (e.g., AlpacaEval2 win rate), making early stopping based on this metric unreliable. This further underscores the need for methods like PEPO that are inherently robust to overoptimization.

During preprocessing, we truncate sequences to a maximum length of 1024 tokens, with prompts limited to 512 tokens to ensure sufficient space for generated responses. 

\paragraph{Data Partitioning.} For \texttt{PEPO}, we partition the training data into $L$ disjoint subsets of approximately equal size, where $L$ is the ensemble size. This partitioning ensures statistical independence across ensemble members, as required by our theoretical analysis (Lemma~\ref{lemma:pessmain}). We note that in practice, data efficiency could potentially be improved by using overlapping subsets, though we leave this exploration to future work.

\subsubsection{Training Procedure}

We train all models using the AdamW optimizer with \texttt{bfloat16} mixed precision for computational efficiency. Training is carried out with a learning rate of $1\times 10^{-5}$ and a weight decay of $0.01$. We maintain a consistent effective batch size of 64 across all experiments, using gradient accumulation where necessary. To accelerate training, we pre-compute and cache the reference model's log-probabilities for all training examples, avoiding redundant forward passes through the frozen backbone.

For \texttt{PEPO}, the $\beta$ parameter controlling KL regularization is set to 0.1, matching the DPO baseline. We set the pessimism parameter $\lambda$ (denoted as $\alpha$ in our codebase) to 0.1. While our theory suggests this term should be state action-dependent, we treat it as a constant prompt-response independent hyperparameter in practice for simplicity.
We choose to take the $\lambda$ as a constant because we see the minimum over the ensemble as the main pessimistic mechanism.
We also conducted an ablation for $\lambda$, but observed no substantial difference in downstream performance. We ablate the ensemble size $L \in \{2, 3, 4\}$ to understand the trade-off between pessimism strength and data efficiency. Note that $L=1$ corresponds to a single model trained with the pessimistic DPO loss, while $L=1$ with the pessimism parameter $\alpha$ set to 0 recovers the standard DPO.

For the baselines, $\chi^2$PO uses $\beta=0.1$, $\gamma=1.0$, and $R_{\max}=10.0$ with internal gradient clipping, following the original paper. SFT+DPO combines the DPO objective with a supervised fine-tuning regularizer weighted by $\lambda_{\text{SFT}}=0.005$. In SFT+DPO we use chosen as well as rejected answers from the UltraFeedback as an estimate for $\pidata$. All experiments use seed 42 for reproducibility.

\subsubsection{Ensemble Training and Inference}

The \texttt{PEPO} ensemble is trained by assigning each LoRA adapter to its corresponding data partition. Ensemble members can be trained sequentially or in parallel across multiple GPUs \textbf{without any required communication} between the members of the ensemble. 

While we introduced a rejection sampling scheme in \Cref{sec:algo} that samples exactly from $\piout$, it can be computationally expensive at inference time. For our large-scale experiments, we instead employ the token-level approximation detailed in \Cref{app:tokenpepo}. This variant applies the pessimistic aggregation at each decoding step, enabling efficient autoregressive generation. We use greedy decoding (temperature $T=0$) with this token-level aggregation for all reported results unless stated otherwise.

\paragraph{Shared backbone inference} During inference of \texttt{PEPO}, we load multiple adapters together with one shared backbone. This allows us to efficiently switch the adapters on and off to run the forward pass with different members of the ensemble without reloading the model to and from device each time. We make use of multiple GPUS by distributing multiple copies of this shared backbone model on each available device.

\subsubsection{Evaluation Protocol}

We evaluate all models on AlpacaEval2~\citep{alpaca_eval}, a benchmark that measures pairwise win rate against a reference model across 805 diverse prompts spanning open-ended generation, reasoning, and instruction-following tasks. We report the win rate against the initial (epoch 0) SFT model, which measures the improvement from preference optimization. To ensure reproducibility and manage costs, we use Llama-3-70B-Instruct as the preference judge (run locally), utilizing the \texttt{alpaca\_eval\_llama3\_70b\_fn} configuration.

All responses are generated using a maximum of 1024 new tokens. We evaluate model checkpoints at every training epoch to capture learning dynamics and potential over-optimization. The win rates reported in the \Cref{tab:best-winrates} correspond to the epoch checkpoint generated by each algorithm.

\subsubsection{Computational Resources}

Experiments were conducted on NVIDIA GH200 120GB and NVIDIA A100-SXM4-80GB GPUs. We allocated one GPU per ensemble member for training (e.g., 4 GPUs for a \texttt{PEPO} model with $L=4$, trained in parallel) and two GPUs for evaluation and generation (mainly to accommodate the Llama-3-70B-Instruct judge). Since we parallelize the training of the ensemble across devices and partition the dataset such that each member sees only $1/L$ of the data, the wall-clock time required for training is significantly reduced (roughly proportional to $1/L$) compared to training a single model on the full dataset. 

For training, we observed that the average wall-clock time per epoch for the Llama-3-8B model dropped from 3.61 hours ($L=1$) to 0.90 hours ($L=4$). Similarly, for Yi-34B, training took approximately 4.0 hours per epoch with $L=4$. 

For evaluation on AlpacaEval2, generating responses and computing win rates using the token-level sampling detailed in \Cref{app:tokenpepo} for the 8B model takes approximately 30 minutes on a single NVIDIA GH200 GPU and approximately 1.25 hours on an NVIDIA A100 GPU. In contrast, full rejection sampling using the minimum probability criterion is significantly slower. We observe that the acceptance probability decreases towards $0$ as the ensemble is trained for more epochs. This causes the generator to saturate the trial limit ($\sim 16$ attempts per answer), increasing the evaluation time to approximately 80 hours on an NVIDIA A100 GPU. However, an optimized variant using mean and standard deviation criteria mitigates this collapse, maintaining a high acceptance rate ($\alpha \approx 0.8$) and low overhead ($\sim 1.2$ attempts per answer). While this optimized variant reduces the runtime to around 4 hours, it remains significantly slower ($8\times$) than the greedy approach. 
\subsubsection{Generation time win rate tradeoff for different ensembles}
We report the same analysis reported in \Cref{tab:refsamp_vs_token} for epochs $2$ and $3$ of the model Llama-8B. The same pattern observed for epoch $1$  in \Cref{tab:refsamp_vs_token} repeats here.
\begin{figure}[ht]
    \centering
    \begin{tikzpicture}
        \begin{scope}[xshift=0cm]
            \begin{axis}[
                scale=0.5,
                xlabel={ms/token},
                ylabel={Win Rate (\%)},
                grid=major,
                xmin=0, xmax=180,
                ymin=45, ymax=78,
                legend style={at={(0.98,0.02)}, anchor=south east, font=\scriptsize},
            ]
                \addplot[only marks, mark=triangle*, mark size=4pt, color=gray, thick] coordinates {(6.221, 50.00)};
                \addplot[only marks, mark=square*, mark size=4pt, color=teal, thick] coordinates {(24.089, 65.35)};
                \addplot[only marks, mark=*, mark size=4pt, color=red!70!black, thick] coordinates {(101.012, 68.08)};
                \legend{DPO, \texttt{PEPO} token, \texttt{PEPO} rejection}
            \end{axis}
        \end{scope}
        \begin{scope}[xshift=5.5cm]
            \begin{axis}[
                scale=0.5,
                xlabel={ms/token},
                grid=major,
                xmin=0, xmax=180,
                ymin=45, ymax=78,
            ]
                \addplot[only marks, mark=triangle*, mark size=4pt, color=gray, thick] coordinates {(6.380, 50.00)};
                \addplot[only marks, mark=square*, mark size=4pt, color=teal, thick] coordinates {(23.812, 65.29)};
                \addplot[only marks, mark=*, mark size=4pt, color=red!70!black, thick] coordinates {(162.416, 69.11)};
            \end{axis}
        \end{scope}
    \end{tikzpicture}
    \caption{Win rate vs.\ generation latency trade-off for epoch~2 (left) and epoch~3 (right). Token-level \texttt{PEPO} achieves win rates close to rejection sampling at a fraction of the computational cost.\label{fig:tradeoff}}
\end{figure}
\newpage
\section{Guarantees for the token level implementation of \texttt{PEPO}}
\label{app:tokenpepo}
As mentioned in the main text, we also introduce a version for \texttt{PEPO} which implements the pessimistic aggregation rule given by \Cref{lemma:closed_form} at the token level rather than the trajectory level. This approach is faster at generation time because it does not reject generations, while \Cref{alg:rejsampl} does. In practice, this translates to a 3x speed at generation time, therefore we preferred this version for most of our large-scale evaluations.

We can notice that the token level aggregation rule is even more pessimistic than the trajectory level version and hence avoids dependence on $C^{\mathrm{all}}$. 

For this result, we consider the following token-level MDP. We consider that it has horizon (sequence length) equal to $H \in \mathbb{N}$.
 As in the main paper we denote as $x$ the initial user prompt and as $a$ the model response. We denote the individual tokens of the response as $\bcc{a_h}^H_{h=1}$ and use the squared brackets to denote concatenation $a = [a_1, \dots, a_H]$. At each stage $h\in[H]$, we consider the state, $x_h = [x,a_1, \dots, a_{h-1}]$, which is the concatenation of the initial prompt and the tokens the model produced so far.
 We consider that $\piref$ and $\tilde{\pi}^\ell$ are autoregressive models. That is, $\piref(a|x) = \prod^{H-1}_{h=1} \piref(a_{h+1} |x, a_1, \dots, a_{h})$ and $\tilde{\pi}^\ell(a|x) = \prod^{H-1}_{h=1} \tilde{\pi}^\ell(a_{h+1} |x, a_1, \dots, a_{h})$. 
 
  At this point, we introduce the token-estimated gap $$\widehat{\Delta}^{\mathrm{token}}(x,a,b):= \beta \sum^H_{h=1}\br{\min_{\ell\in[L]} \log \frac{\tilde{\pi}^\ell(a_h|x_h)}{\piref(a_h|x_h)} - \max_{\ell\in[L]} \log \frac{\tilde{\pi}^\ell(b_h|x_h)}{\piref(b_h|x_h)}},$$ which will be crucial in proving the following result.
\begin{lemma}
\label{lem:pepotoken}
    Let, $\piout^{\mathrm{token}}(a_h|x_h) = \frac{\min_{\ell\in[L]}\tilde{\pi}^{\ell}(\cdot|x, a_1,\dots,a_{h-1})}{\sum_{a\in\A_h}\min_{\ell\in[L]}\tilde{\pi}^{\ell}(a|x, a_1,\dots,a_{h-1})}$, then it holds that with probability $1-\delta$, for any $\beta>0$, it holds that
    $J_\beta(\pi^\star) - J_\beta(\piout^{\mathrm{token}}) \leq \sqrt{C^\star \innerprod{\pidata}{\mathrm{Err}^2_{\mathrm{token}}}}$,
where $\mathrm{Err}_{\mathrm{token}}(x,a) :=  \widehat{\Delta}^{\mathrm{token}}(x,a,\pidata) - \Delta_{r^\star}(x,a,\pidata)$.
\end{lemma}
\begin{proof}
The proof strategy is to prove that the token-level approach introduces even more pessimism. That is, 
\[
\widehat{\Delta}^{\mathrm{token}}(x,a,b) \leq \widehat{\Delta}(x,a,b) \quad \forall x,a,b \in \X\times\A\times\A.
\]
Towards this end, let us recall that 
\[
\widehat{\Delta}(x,a,b) = \beta \min_{\ell \in [L]} \log \frac{\tilde{\pi}^\ell(a|x)}{\piref(a|x)} - \beta \max_{\ell \in [L]} \log \frac{\tilde{\pi}^\ell(b|x)}{\piref(b|x)},
\]
and we start lower bounding the first term as follows.
\begin{align*}
\beta \min_{\ell \in [L]} \log \frac{\tilde{\pi}^\ell(a|x)}{\piref(a|x)}
&= \beta \min_{\ell \in [L]} \log \prod^{H-1}_{h=1} \frac{\tilde{\pi}^\ell(a_{h+1}|x, a_1, \dots, a_h)}{\piref(a_{h+1}|x, a_1, \dots, a_h)} \\
&= \beta \min_{\ell \in [L]} \sum^{H-1}_{h=1} \log  \frac{\tilde{\pi}^\ell(a_{h+1}|x, a_1, \dots, a_h)}{\piref(a_{h+1}|x, a_1, \dots, a_h)} \\
&\geq  \sum^{H-1}_{h=1} \beta \min_{\ell \in [L]} \log  \frac{\tilde{\pi}^\ell(a_{h+1}|x, a_1, \dots, a_h)}{\piref(a_{h+1}|x, a_1, \dots, a_h)}  \\
\end{align*}
Then, we can lower bound the second term as follows.
\begin{align*}
- \beta \max_{\ell \in [L]} \log \frac{\tilde{\pi}^\ell(b|x)}{\piref(b|x)}
&= -\beta \max_{\ell \in [L]} \log \prod^{H-1}_{h=1} \frac{\tilde{\pi}^\ell(b_{h+1}|x, b_1, \dots, b_h)}{\piref(b_{h+1}|x, b_1, \dots, b_h)} \\
&= -\beta \max_{\ell \in [L]} \sum^{H-1}_{h=1} \log  \frac{\tilde{\pi}^\ell(b_{h+1}|x, b_1, \dots, b_h)}{\piref(b_{h+1}|x, b_1, \dots, b_h)} \\
&\geq  -\sum^{H-1}_{h=1} \beta \max_{\ell \in [L]} \log  \frac{\tilde{\pi}^\ell(b_{h+1}|x, b_1, \dots, b_h)}{\piref(b_{h+1}|x, b_1, \dots, b_h)}  \\
\end{align*}
We can conclude that 
\[
\widehat{\Delta}^{\mathrm{token}}(x,a,b) \leq \widehat{\Delta}(x,a,b) \quad \forall x,a,b \in \X\times\A\times\A.
\]
At this point, let $r^{\mathrm{token}}(x, a) = \abs{\mathcal{D}^\ell}^{-1}\sum_{A^- \in \cup_{\ell\in[L]}\mathcal{D}^\ell} \hat{\Delta}^{\mathrm{token}}(x,a,A^{-})$, with the same steps as in \Cref{lemma:closed_form}, we can prove that the closed solution of $$\max_{\pi \in \Pi_{\mathrm{AR}} } \innerprod{\pi}{r^{\mathrm{token}}} - \beta KL(\pi,\piref)$$ over a the class of autoregressive models  $\Pi_{\mathrm{AR}}$ is given by $$\piout^{\mathrm{token}}(a_{h+1}|x_{h+1}) = \frac{\min_{\ell\in[L]}\tilde{\pi}^{\ell}(a_{h+1}|x_{h+1})}{\sum_{a\in\A}\min_{\ell\in[L]}\tilde{\pi}^{\ell}(a_{h+1}|x_{h+1})}.$$ The policy $\piout^{\mathrm{token}}$ can now be computed in closed form as the denominator only involves a sum over tokens and not responses. 

Finally, since $$
\widehat{\Delta}^{\mathrm{token}}(x,a,b) \leq \widehat{\Delta}(x,a,b) \quad \forall x,a,b \in \X\times\A\times\A,$$ and that $ \widehat{\Delta}(x,a,b) \leq \Delta_{r^\star}(x,a,b) \quad \forall x,a,b \in \X\times\A\times\A$ with probability at least $1-\delta$ according to Lemma~\ref{lemma:pessmain}. We have that $ \widehat{\Delta}(x,a,b) \leq \Delta_{r^\star}(x,a,b)$ with probability $1-\delta$. Then, invoking Lemma~\ref{lemma:conversion} with $\widehat{\Delta}^{\mathrm{token}}$ instead of $\widehat{\Delta}$, we have that with probability $1-\delta$,
\[
J_\beta(\pi^\star) - J_\beta(\piout^{\mathrm{token}}) \leq \sqrt{C^\star \innerprod{\pidata}{\mathrm{Err}^2_{\mathrm{token}}}}.
\]
\end{proof}
\newpage
\section{Analysis of \texttt{PEPO}}
\label{app:analysis}
This section presents the analysis of \texttt{PEPO}.
\subsection{Proof of \Cref{lemma:closed_form}: closed form solution of \Cref{eq:outputDPO} }
\begin{proof}
    The first order optimality conditions for \Cref{eq:outputDPO} give
$$
r^-(x,a) = \beta \log \br{\frac{\piout(a|x)}{\piref(a|x)}} + \beta \log Z(x)
$$
where $Z(x)$ needs to be chosen to guarantee normalization of $\piout$.
Plugging in, the definition of $r^-(x,a)$, we obtain that 
$$
\beta \min_{\ell \in [L]} \br{ \log \frac{\tilde{\pi}^{\ell}(a|x)}{\piref(a|x)} - \zeta^\ell(x)} = \beta \log \br{\frac{\piout(a|x)}{\piref(a|x)}} + \beta \log Z(x)  
$$
$$
\implies \min_{\ell \in [L]} \frac{\tilde{\pi}^{\ell}(a|x) \exp\br{-\zeta^\ell(x)}}{\piref(a|x)} = \frac{\piout(a|x) Z(x)}{\piref(a|x) }.
$$
Then, choosing $Z(x)$ to ensure that $\sum_{a\in \A}\piout(a|x)=1$ for all $x\in \X$, we obtain that 
$$
\piout(a|x) = \frac{\min_{\ell\in[L]}\tilde{\pi}^{\ell}(a|x) \exp\br{-\zeta^\ell(x)}}{\sum_{a\in\A}\min_{\ell\in[L]}\tilde{\pi}^{\ell}(a|x) \exp\br{-\zeta^\ell(x)}}.
$$
\end{proof}

\subsection{Proof of Lemma~\ref{lemma:pessmain}: Establishing pessimism}
\pessmain*
\begin{proof}
    Let us denote $\widehat{\mathbb{P}}^\ell(x, a \succ b; \pi) = \sigma\br{\beta \Delta_{\log \pi /\piref}(x,a,b) + \lambda(x,a,b)}$ and let us notice that we can rewrite the DPO objective (see \Cref{eq:RegDPO}) for each $\ell \in [L]$ as
    \begin{align*}
        &J_{\mathrm{pessDPO}}(\pi;\cD^\ell) = \sum_{X,A^+,A^-\in \cD^\ell} \log \sigma \br{\widehat{\mathbb{P}}^\ell(X, A^+ \succ A^-; \pi)} \\
        &= \sum_{x,a\in \X\times\A} \sum_{b > a} \log \sigma \br{\widehat{\mathbb{P}}^\ell(x,a \succ b; \pi)} \mathds{1}_{\bcc{X,A^+,A^-=x,a,b}} + \log \sigma \br{\widehat{\mathbb{P}}^\ell(x,b \succ a; \pi)} \mathds{1}_{\bcc{X,A^+,A^-=x,b,a}} \\
        &= \sum_{x,a\in \X\times\A} \sum_{b > a} \log \sigma \br{\widehat{\mathbb{P}}^\ell(x,a \succ b; \pi)} \mathds{1}_{\bcc{X,A^+,A^-=x,a,b}} \\&\phantom{=}+ \log \sigma \br{1 -\widehat{\mathbb{P}}^\ell(x,a \succ b;  \pi)} \mathds{1}_{\bcc{X,A^+,A^-=x,b,a}},
    \end{align*}
    where we use the notation $b > a$ to indicate all actions $b$ that follow action $a$ for some canonical ordering of the action space.
    Then, taking the first derivative and equating it to zero, we obtain
    \begin{equation*}
\frac{N^\ell(x, a\succ b)}{\widehat{\mathbb{P}}^\ell(x,a \succ b;  \tilde{\pi}^\ell)} = \frac{N^\ell(x, b\succ a)}{1- \widehat{\mathbb{P}}^\ell(x,a \succ b;  \tilde{\pi}^\ell)} \quad \forall ~~x,a,b \in \X\times \A \times \A,
    \end{equation*}
where we defined the counts $N^\ell(x,a,b) = \sum_{X,A^+,A^- \in \mathcal{D}^\ell} \mathds{1}_{\bcc{X,A^+,A^-=x,a,b}} + \mathds{1}_{\bcc{X,A^+,A^-=x,b,a}} $ and $N^\ell(x,a \succ b) = \sum_{X,A^+,A^- \in \mathcal{D}^\ell} \mathds{1}_{\bcc{X,A^+,A^-=x,a,b}}$. Then, rearranging we obtain,
\[
\widehat{\mathbb{P}}^\ell(x,a \succ b;  \tilde{\pi}^\ell) = \frac{N^\ell(x,a\succ b)}{N^\ell(x,a,b)} \quad \text{if} ~~N^\ell(x,a,b) > 0,
\]
and arbitrary otherwise. Then, using that $\widehat{\mathbb{P}}^\ell(x,a \succ b;  \pi) : = \sigma\br{\beta \Delta_{\log \pi /\piref}(x,a,b) + \lambda(x,a,b)} = $ and that $\lambda(x,a,b) = \sigma^{-1}(\nicefrac{N^\ell(x,a\succ b)}{N^\ell(x,a,b)}) - \sigma^{-1}(\nicefrac{N^\ell(x,a\succ b)}{(N^\ell(x,a,b)+2)} ),$ we obtain that
\begin{equation*}
    \beta \Delta_{\log \tilde{\pi}^\ell /\piref}(x,a,b) = \sigma^{-1}\br{\frac{N^\ell(x,a\succ b)}{N^\ell(x,a,b)}} - \lambda(x,a,b) = \sigma^{-1}\br{\frac{N^\ell(x,a\succ b)}{N^\ell(x,a,b)+2}}.
\end{equation*}
At this point, invoking \citep[Lemma 3]{cassel2025batch} we have that for any fixed state $x \in \X$, actions $a,b \in \A\times\A$ and any possible value of $N^\ell(x,a,b)$, it  holds that
    $$
    \mathbb{P}\bs{\min_{\ell \in [L]} \frac{N^\ell(x,a\succ b)}{N^\ell(x,a,b)+2} \geq \mathbb{P}^{\mathrm{true}}_{r^\star}(a\succ b| x) } \leq  e^{-2L/7}.
    $$
Therefore, choosing $L \geq \frac{7 \log(\abs{\X}\abs{\A}^2/\delta)}{2}$, we obtain that with probability at least $1-\delta$ it holds that
$$
    \mathbb{P}\bs{\exists x,a,b \in \X\times\A\times\A ~~\min_{\ell \in [L]} \frac{N^\ell(x,a\succ b)}{N^\ell(x,a,b)+2} \geq \mathbb{P}^{\mathrm{true}}_{r^\star}(a\succ b| x) } \leq \delta.
    $$
Then, passing to the complementary event and using that $\sigma^{-1}(\cdot)$ is an increasing function, we obtain that
$$
    \mathbb{P}\bs{\beta \min_{\ell \in [L]} \Delta_{\log \tilde{\pi}^\ell/\piref}(x,a,b)  \leq \Delta_{r^\star}(x,a,b) ~~~ \forall x,a,b \in \X\times\A\times\A ~~ } \leq \delta.
    $$
    Finally, notice that for each $\ell \in [L]$
    \begin{align*}
        \Delta_{\log \tilde{\pi}^\ell/\piref}(x,a,b) &= \beta \log \frac{\tilde{\pi}^\ell(a|x)}{\piref(a|x)} - \beta \log \frac{\tilde{\pi}^\ell(b|x)}{\piref(b|x)} \\
        &= \beta \log \frac{\tilde{\pi}^\ell(a|x)}{\piref(a|x)} - \zeta^\ell(x) - \br{\beta \log \frac{\tilde{\pi}^\ell(b|x)}{\piref(b|x)} - \zeta^\ell(x)}  \\
         &\geq \min_{\ell\in [L]}\br{\beta \log \frac{\tilde{\pi}^\ell(a|x)}{\piref(a|x)} - \zeta^\ell(x)} - \max_{\ell'\in[L]}\br{\beta \log \frac{\tilde{\pi}^{\ell'}(b|x)}{\piref(b|x)} - \zeta^{\ell'}(x)}  \\
         &= \widehat{\Delta}(x,a,b).
    \end{align*}
    Then, by monotonicity of expectation we have that for any $x,a \in \X\times\A$,
    \begin{align*}\Delta_{\log \tilde{\pi}^\ell/\piref}(x,a,\pidata) &= \sum_{b \in \A}\pidata(b|x) \Delta_{\log \tilde{\pi}^\ell/\piref}(x,a,b) \\&\geq \sum_{b \in \A}\pidata(b|x) \widehat{\Delta}(x,a,b) = \widehat{\Delta}(x,a,\pidata).
    \end{align*}
    This concludes the proof.\end{proof}
\subsection{Proof of Lemma~\ref{lemma:conversion}: from suboptimality to estimation error.}
At this point, we can present the upper bound on the suboptimality, i.e. $J_\beta(\pi^\star) - J_\beta(\piout)$, which depends only on the single policy concentrability and the generalization error of the estimated gap $\widehat{\Delta}$. Such upper bound holds under the pessimistic event defined as follows:
$$
\mathcal{E}_{\mathrm{pessimism}} = \bcc{  \sum_{b}\pidata(b|x)  \widehat{\Delta}(x,a,b) \leq \sum_b \pidata(b|x) \Delta_{r^\star}(x,a,b)  ~~~~\forall x,a \in \X\times\A}.
$$
The result is proven formally in the following after restating the Lemma for convenience.
\conversion*
\begin{proof}
The proof follows from these calculations.
\begin{align}
\innerprod{\pi^\star - \piout}{r^\star} &= \sum_{x}\initial(x) \sum_a (\pi^\star(a|x) - \piout(a|x)) r^\star(x,a) \nonumber \\
&= \sum_{x}\initial(x) \sum_{a,b} (\pi^\star(a|x) - \piout(a|x)) \pidata(b|x) (r^\star(x,a) - r^\star(x,b))\nonumber\\
&= \sum_{x}\initial(x) \sum_{a,b} (\pi^\star(a|x) - \piout(a|x)) \pidata(b|x) (\Delta_{r^\star}(x,a,b))\nonumber\\
&= \sum_{x}\initial(x) \sum_{a,b} (\pi^\star(a|x) - \piout(a|x)) \pidata(b|x) (\Delta_{r^\star}(x,a,b) - \widehat{\Delta}(x,a,b))\nonumber\\
&\phantom{=}+ \sum_{x}\initial(x) \sum_{a,b} (\pi^\star(a|x) - \piout(a|x)) \pidata(b|x) \widehat{\Delta}(x,a,b) \label{eq:second_term}
\end{align}
Looking at the first term, we obtain that
\begin{align*}
    \sum_{x}\initial(x) &\sum_{a,b} (\pi^\star(a|x) - \piout(a|x)) \pidata(b|x) (\Delta_{r^\star}(x,a,b) - \widehat{\Delta}(x,a,b)) \\&= \sum_{x}\initial(x) \sum_{a,b}\pi^\star(a|x) \pidata(b|x) (\Delta_{r^\star}(x,a,b) - \widehat{\Delta}(x,a,b)) \\
    &\phantom{=} + \sum_{x}\initial(x) \sum_{a,b}\piout(a|x) \pidata(b|x) (\widehat{\Delta}(x,a,b) -\Delta_{r^\star}(x,a,b) ) \\
    &\leq \sum_{x}\initial(x) \sum_{a,b}\pi^\star(a|x) \pidata(b|x) (\Delta_{r^\star}(x,a,b) - \widehat{\Delta}(x,a,b))  
    \quad \text{(Using $\mathcal{E}_{\mathrm{pessimism}}$)}\\
    &=\sum_{x}\initial(x) \sum_{a,b}\frac{\pi^\star(a|x)}{\sqrt{\pidata(a|x)}} \sqrt{\pidata(a|x)}\pidata(b|x) (\Delta_{r^\star}(x,a,b) - \widehat{\Delta}(x,a,b)) \\
    &\leq \sqrt{\br{\sum_{x}\initial(x) \sum_{a}\frac{(\pi^\star(a|x))^2}{\pidata(a|x)}} \cdot \sum_{a}\pidata(a|x) (\sum_b \pidata(b|x) (\Delta_{r^\star}(x,a,b) - \widehat{\Delta}(x,a,b)))^2} 
\end{align*}
Finally, we upper bound \Cref{eq:second_term} by the fact that
\begin{align*}
\piout = \argmax_{\pi\in\Pi} \sum_x \initial(x)\sum_a \pi(a|x) r^-(x,a) - \beta D_{\mathrm{KL}}(\pi, \piref)
\end{align*}
where $D_{\mathrm{KL}}(\pi,\pi') = \sum_x \initial(x) \sum_{a} \pi(a|x) \log \frac{\pi(a|x)}{\pi'(a|x)}$ and
\begin{align*}r^-(x,a) &:= \widehat{\Delta}(x,a,b)  + \beta \max_{\ell \in [L]}\br{ \log \frac{\tilde{\pi}^\ell(b|x)}{\piref(b|x)} - \zeta^\ell(x)} \\
& = \beta \min_{\ell \in [L]}\br{ \log \frac{\tilde{\pi}^\ell(a|x)}{\piref(a|x)} - \zeta^\ell(x)}.
\end{align*}

Indeed, notice that
\begin{align*}
&\sum_{x}\initial(x) \sum_{a,b} (\pi^\star(a|x) - \piout(a|x)) \pidata(b|x)\widehat{\Delta}(x,a,b) =\\
    &\sum_{x}\initial(x) \sum_{a,b} (\pi^\star(a|x) - \piout(a|x)) \pidata(b|x) r^-(x,a) \\
    &\phantom{=}-\underbrace{\sum_{x}\initial(x) \sum_{a,b} (\pi^\star(a|x) - \piout(a|x)) \pidata(b|x) \beta \max_{\ell\in[L]} \br{ \log \frac{\tilde{\pi}^\ell(b|x)}{\piref(b|x)} - \zeta^\ell(x)} }_{=0}\\
    & = \sum_{x}\initial(x) \sum_{a} (\pi^\star(a|x) - \piout(a|x)) r^-(x,a) \\
    &\leq \beta D_{\mathrm{KL}}(\pi^\star,\piref) - \beta D_{\mathrm{KL}}(\piout,\piref)
\end{align*}
Where the last inequality follows from the optimality of $\piout$.
Then, putting all together 
we get
\begin{align*}
\innerprod{\pi^\star - \piout}{r^\star} &+ \beta D_{\mathrm{KL}}(\piout,\piref) - \beta D_{\mathrm{KL}}(\pi^\star,\piref) \leq \sqrt{\sum_{x}\initial(x) \sum_{a}\frac{(\pi^\star(a|x))^2}{\pidata(a|x)}} \\&\cdot \sqrt{\sum_{a}\pidata(a|x)(\sum_b \pidata(b|x)  (\Delta_{r^\star}(x,a,b) - \widehat{\Delta}(x,a,b)))^2},
\end{align*}
which implies the statement.

\textbf{Conversion with fast rates}
Now, we aim to prove the fast rate , i.e. $1/N$ type, result. This is enable by a different suboptimality to estimation error conversion. 
To this end, let us notice that 
\begin{itemize}
    \item $\piout$ is optimal for the RLHF problem (see \eqref{eq:goal})  with reward function $\widehat{\Delta}(x,a,\pidata)$.
    \item $\pi^\star$ is optimal for the RLHF problem  with reward function $\Delta_{r^\star}(x,a,\pidata)$.
    \item Thanks, to Lemma~\ref{lemma:pessmain}, we have that with high probability $\Delta_{r^\star}(x,a,\pidata) \geq \widehat{\Delta}(x,a,\pidata)$ for any $x,a \in \X\times\A$.
\end{itemize}
Under the conditions above, we can invoke \citep[Lemma C.2]{ji2026optimal} to guarantee that
\begin{align*}
    &J_\beta(\pi^\star) - J_\beta(\piout) \leq
    \beta^{-1}\sum_{x\in\X} \sum_{a\in \A} \initial(x) \pi^\star(a|x) \br{\sum_{b \in \A} \pidata(b|x)(\Delta_{r^\star}(x,a, b) - \widehat{\Delta}(x,a,b))^2} \\
    &=
    \beta^{-1} \sum_{x\in\X} \sum_{a\in \A} \initial(x) \frac{\pi^\star(a|x)}{\pidata(a|x)} \pidata(a|x) \br{\sum_{b \in \A} \pidata(b|x) (\Delta_{r^\star}(x,a, b) - \widehat{\Delta}(x,a,b))^2} \\
    &\leq \beta^{-1} C^{\star}_{\infty}  \sum_{x\in\X} \sum_{a\in \A} \initial(x) \pidata(a|x) \br{\sum_{b \in \A} \pidata(b|x) (\Delta_{r^\star}(x,a, b) - \widehat{\Delta}(x,a,b))}^2  \\
    &= \beta^{-1} C_{\infty}^\star \innerprod{\pidata}{\mathrm{Err}^2}.
\end{align*}
\end{proof}

\subsection{Proof of Lemma~\ref{lemma:concentration_main}: bounding the estimation error}
The goal of this subsection is to provide a high probability bound on the following generalization error 
$$
\innerprod{\pidata}{\mathrm{Err}^2} := \sum_{x\in \X} \initial(x) \sum_{a}\pidata(a|x)(\sum_b \pidata(b|x)  (\Delta_{r^\star}(x,a,b) - \widehat{\Delta}(x,a,b)))^2.
$$
The formal results are presented in Lemma~\ref{lemma:concentration_main}.
\concentrationlem*
\begin{proof}
    Let us notice that denoting for a fixed $x,a, b\in \X\times\A^2$, $$\bar{\ell} = \argmax_{\ell \in [L]} \beta \log \frac{\tilde{\pi}^\ell(b|x)}{\piref(b|x)} - \beta \zeta^\ell(x)$$ and $$\underline{\ell} = \argmin_{\ell \in [L]} \beta \log \frac{\tilde{\pi}^\ell(a|x)}{\piref(a|x)} - \beta\zeta^\ell(x).$$ Then, we have that 
    \begin{align*}
    \abs{\widehat{\Delta}(x,a,b) - \Delta_{r^\star}(x,a,b)} &\leq \abs{\log \frac{\tilde{\pi}^{\underline{\ell}}(a|x)}{\piref(a|x)} - \beta\zeta^{\underline{\ell}}(x) - r^\star(x,a)}
    \\&\phantom{=}+ \abs{\beta \log \frac{\tilde{\pi}^{\bar{\ell}}(b|x)}{\piref(b|x)} - \beta\zeta^{\bar{\ell}}(x) - r^\star(x,b)} \\
    & \leq \sum_{\mathfrak{a} \in \bcc{a,b}}\max_{\ell \in [L]}\abs{\beta \log \frac{\tilde{\pi}^{\ell}(\mathfrak{a}|x)}{\piref(\mathfrak{a}|x)} - \beta\zeta^{\ell}(x) - r^\star(x,\mathfrak{a})}.
    \end{align*}

Then, using the fact that $r^\star(x,\pidata) = 0$ without loss of generality and that by Azuma-Hoeffding inequality, we have that for any action $a \in \A$ and state $x \in \X$
\begin{align*}
&\abs{ \log(\tilde{\pi}^\ell(a|x)/\piref(a|x)) - \zeta^\ell(x) - \underbrace{\br{ \log(\tilde{\pi}^\ell(a|x)/\piref(a|x)) - \sum_{\mathfrak{a} \in \A} \pidata(\mathfrak{a}|x) \log \frac{\tilde{\pi}^\ell(\mathfrak{a}|x)}{\piref(\mathfrak{a}|x)} }}_{:= g(x,a)}} \\&= \abs{\frac{N^{-\ell}(x) (\log(\tilde{\pi}^\ell(a|x)/\piref(a|x)) - \zeta^\ell(x) -  g(x,a))}{N^{-\ell}(x) + 2} - \frac{2(g(x,a)-\log(\tilde{\pi}^\ell(a|x)/\piref(a|x)) + \zeta^\ell(x))}{N^{-\ell}(x) + 2} } \\
&\leq \mathcal{O}\br{\sqrt{\frac{\RMAX^2/\beta^2\log(N \abs{\X}\abs{\A}/\delta)}{N^{-\ell}(x)+2}} + \frac{\RMAX/\beta}{N^{-\ell}(x)+2}},
\end{align*}
where the last inequality follows with probability $1-\delta$ by the Azuma-Hoeffding inequality, 
which can be applied because 
\begin{align*}
   N^{-\ell}(x) &(\log(\tilde{\pi}^\ell(a|x)/\piref(a|x)) - \zeta^\ell(x) -  g(x,a)) \\&= \sum_{X,A \in \cD\setminus\cD^\ell} (\log(\tilde{\pi}^\ell(a|x)/\piref(a|x)) -  g(x,a)) - \log \frac{\tilde{\pi}^\ell(A|X)}{\piref(A|X)} \mathds{1}_{\bcc{X=x}} \\
   &= \sum_{X,A \in \cD\setminus\cD^\ell} \br{\sum_{a\in \A} \pidata(a|x) \log \frac{\tilde{\pi}^\ell(a|x)}{\piref(a|x)}  - \log \frac{\tilde{\pi}^\ell(A|X)}{\piref(A|X)}} \mathds{1}_{\bcc{X=x}},
\end{align*}
is the sum of a martingale difference sequence which is almost surely bounded by $\RMAX/\beta$ because of the constraints imposed on $\tilde{\pi}^\ell$ in \Cref{eq:RegDPO}. For the same reason, we also have that
\[
\abs{g(x,a)-\log(\tilde{\pi}^\ell(a|x)/\piref(a|x)) + \zeta^\ell(x)} \leq \frac{\RMAX}{\beta},
\]
which allows to bound the second term as $\frac{\RMAX/\beta}{N^{-\ell}(x) + 2}$.
Then, noticing that $g(x,a) = \Delta_{\log \tilde{\pi}^\ell/\piref} (x,\mathfrak{a},\pidata)$, we obtain that
\[
\abs{\log(\tilde{\pi}^\ell(a|x)/\piref(a|x)) - \zeta^\ell(x) - \Delta_{\log \tilde{\pi}^\ell/\piref} (x,\mathfrak{a},\pidata)} \leq \mathcal{O}\br{\sqrt{\frac{\RMAX^2/\beta^2\log(N \abs{\X}\abs{\A}/\delta)}{N^{-\ell}(x)+2}} + \frac{\RMAX/\beta}{N^{-\ell}(x)+2}}.
\]
Therefore, using this fact,
\begin{align*}
    &\abs{\widehat{\Delta}(x,a,b) - \Delta_{r^\star}(x,a,b)} \\&\leq  \sum_{\mathfrak{a}\in \bcc{a,b}} \max_{\ell \in [L]} \bigg(\abs{\beta \Delta_{\log \tilde{\pi}^\ell/\piref} (x,\mathfrak{a},\pidata) -  \Delta_{r^\star}(x,\mathfrak{a},\pidata)} \\&\phantom{=} + \abs{\log(\tilde{\pi}^\ell(a|x)/\piref(a|x)) - \zeta^\ell(x) - \Delta_{\log \tilde{\pi}^\ell/\piref} (x,\mathfrak{a},\pidata)}\bigg)\\&\leq 
    \sum_{\mathfrak{a}\in \bcc{a,b}} \max_{\ell \in [L]} \abs{\beta \Delta_{\log \tilde{\pi}^\ell/\piref} (x,\mathfrak{a},\pidata) -  \Delta_{r^\star}(x,\mathfrak{a},\pidata)} \\&+ \mathcal{O}\br{\sqrt{\frac{\RMAX^2\log(N \abs{\X}\abs{\A}/\delta)}{N^{-\ell}(x)+2}} + \frac{\RMAX}{N^{-\ell}(x)+2}}.
\end{align*}
Rearranging the first term, we obtain,
\begin{align*}
    &\sum_{\mathfrak{a}\in \bcc{a,b}} \max_{\ell \in [L]} \abs{\beta \Delta_{\log \tilde{\pi}^\ell/\piref} (x,\mathfrak{a},\pidata) - \Delta_{r^\star}(x,\mathfrak{a},\pidata)} \\&= \sum_{\mathfrak{a}\in \bcc{a,b}} \max_{\ell \in [L]} \abs{ \sum_{a'\in \A} \pidata(a'|x)\bs{\sigma^{-1}\br{\frac{N^\ell(x,\mathfrak{a}\succ a')}{N^\ell(x,\mathfrak{a},a')+2}} - \sigma^{-1}(\mathbb{P}^{\mathrm{true}}_{r^\star}(x,\mathfrak{a}\succ a'))}} \\
    &\leq \sum_{\mathfrak{a}\in \bcc{a,b}} e^{\RMAX} \max_{\ell \in [L]} \sum_{a'\in \A} \pidata(a'|x) \abs{\frac{N^\ell(x,\mathfrak{a}\succ a')}{N^\ell(x,\mathfrak{a},a')+2} - \mathbb{P}^{\mathrm{true}}_{r^\star}(x,\mathfrak{a}\succ a') }
    \\
    &\leq \widetilde{\mathcal{O}}\br{\sum_{\mathfrak{a}\in \bcc{a,b}} e^{\RMAX} \max_{\ell \in [L]} \sum_{a'\in \A} \pidata(a'|x) \sqrt{\frac{\log(\abs{\A}\abs{\X} N /\delta)}{N^\ell(x,\mathfrak{a},a')+2}} }
    \\
    &\leq \widetilde{\mathcal{O}}\br{\sum_{\mathfrak{a}\in \bcc{a,b}} e^{\RMAX} \sum_{a'\in \A} \pidata(a'|x) \sqrt{\frac{L \log(\abs{\A}\abs{\X} N /\delta)}{N(x,\mathfrak{a},a')+L}} },
\end{align*}
where the second last inequality holds with probability $1-\delta$ thanks to the Azuma-Hoeffding inequality and a union bound over $[N]\times\A\times\X$, whereas the last one follows from $N^\ell(x,\mathfrak{a},a') \geq N(x,\mathfrak{a},a')/L -1$. Moreover, we also used that the Lipschitz constant of $\sigma(\cdot)$ in $[-\RMAX, \RMAX]$ is $\mathcal{O}\br{e^{\RMAX}}$, see e.g. \citep[Lemma F.5]{huang2024correcting}.
Hence, putting all together, we obtain that with probability $1-2\delta$,
\begin{align*}
    \abs{\widehat{\Delta}(x,a,b) - \Delta_{r^\star}(x,a,b)} &\leq 
\widetilde{\mathcal{O}}\br{\sum_{\mathfrak{a}\in \bcc{a,b}} e^{\RMAX} \sum_{a'\in \A} \pidata(a'|x) \sqrt{\frac{L \log(\abs{\A}\abs{\X} N /\delta)}{N(x,\mathfrak{a},a')+L}} }
     \\&+ \widetilde{\mathcal{O}}\br{\sqrt{\frac{\RMAX^2\log(N \abs{\X}\abs{\A}/\delta)}{N(x)+1}}},
\end{align*}
where we have also used that $N^{-\ell}(x) \geq N(x) - 1$.
Then, towards controlling 
\begin{align*} 
\innerprod{\pidata}{\mathrm{Err}^2} &\leq \sum_{x\in\X}\sum_{a,b \in \A\times\A} \initial(x) \pidata(a|x) \pidata(b|x)\abs{\widehat{\Delta}(x,a,b) - \Delta_{r^\star}(x,a,b)}^2 \\
&\leq \widetilde{\mathcal{O}}\br{e^{\RMAX} \sum_{x\in\X}\sum_{a,b \in \A\times\A}\initial(x) \pidata(a|x)  \pidata(b|x)
\frac{L \log(\abs{\A}\abs{\X} N /\delta)}{N(x,a,b)+L} 
},
\end{align*}
where we also used that $N(x) \geq N(x,a,b)$. Therefore, the terms which decays as $1/N(x) $ are lower order terms absorbed into the $\widetilde{\mathcal{O}}$ notation.
Now, using \citep[Lemma D.4]{rosenberg2020near} we have that for any sequence of positive i.i.d.  random variables $Z_1, \dots, Z_n \in [0, B]$ with mean $\mu$ it holds that
$$
N \mu \leq 2 \sum^N_{i=1} Z_i + 4 B \log ( N \delta^{-1}).
$$
In our case, the bounds does not apply directly because $N(\cdot,\cdot,\cdot) : \X\times\A\times\A \rightarrow [0, N]$ is a random function. However, for any $\mathfrak{f} : \X\times\A\times\A \rightarrow [0, N]$, we can apply \citep[Lemma D.4]{rosenberg2020near} for $Z_i = (\mathfrak{f}(X_{i}, A_{i}, A'_{i} ))$, where we recall that $\cD = \bcc{X_{i}, A_{i}, A'_{i}}^N_{i=1}$, and with failure probability $\frac{\delta}{N^{\abs{\X}\abs{\A}^2}}$, we obtain that
\begin{align*}
\sum_{x\in\X}\sum_{a,b \in \A\times\A}
\frac{\initial(x) \pidata(a|x)  \pidata(b|x)}{\mathfrak{f}(x,a,b)+L}  \leq 2 \sum^N_{i=1} \frac{1}{\mathfrak{f}(X_{i}, A_{i}, A'_{i} ) + L} + 4 \abs{\X}\abs{\A}^2 L \log (N^2\delta^{-1}).
\end{align*}
Therefore, choosing specifically $\mathfrak{f}(\cdot,\cdot,\cdot) = N(\cdot,\cdot,\cdot)$ we obtain
\begin{align*}
N \sum_{x\in\X}&\sum_{a,b \in \A\times\A}
\frac{\initial(x) \pidata(a|x)  \pidata(b|x)}{N(x,a,b)+L}  \leq 2 \sum^N_{i=1} \frac{1}{N(X_{i}, A_{i}, A'_{i} ) + L} + 4 \abs{\X}\abs{\A}^2 L \log (N^2\delta^{-1}) \\
&= 2 \sum_{x,a,b \in \X\times\A\times\A} \frac{\sum^N_{i=1}\mathds{1}_{\bcc{X_{i}, A_{i}, A'_{i} = x,a,b }}}{N(x,a,b) + L} + 4 \abs{\X}\abs{\A}^2 L \log (N^2\delta^{-1}) \\
&\leq 2 \sum_{x,a,b \in \X\times\A\times\A} \frac{N(x,a,b)}{N(x,a,b) + L} + 4 \abs{\X}\abs{\A}^2 L \log (N^2\delta^{-1}) \\
&= 2 \abs{\X}\abs{\A}^2 + 4 \abs{\X}\abs{\A}^2 L \log (N^2\delta^{-1})
\\
&= \mathcal{O}\br{ \abs{\X}\abs{\A}^2 L \log (N^2\delta^{-1})}.
\end{align*}
Therefore, we can conclude that
with probability $1-\delta$, it holds that
\begin{align*}
    \innerprod{\pidata}{\mathrm{Err}^2} \leq \widetilde{\mathcal{O}}\br{ \frac{e^{\RMAX} \abs{\X}\abs{\A}^2 L^2 \log (N^2\delta^{-1})}{N}}.
\end{align*}
This concludes the proof.
\end{proof}

\subsection{Technical Results}
In this appendix, we provide the proofs of the auxiliary results we used throughout our analysis. We start by presenting some known results from the literature that we used in our analysis.
\begin{lemma}
    \textbf{Azuma-Hoeffding inequality} Let $\bcc{X_k}^K_{k=1}$ be a martingale difference sequence adapted to the filtration $\mathcal{F}_k$ which is almost surely bounded by $B$, i.e. $\abs{X_k} \leq B$ for all $k \in [K]$.
    Then, we have that
    $$
    \mathbb{P}\bs{\abs{\sum^K_{k=1} X_k} \geq \sqrt{2 K B^2 \log(2/\delta)}} \leq \delta.
    $$
\end{lemma}

\begin{lemma}
    \textbf{Variant of the Freedman's inequality \citep[Lemma D.4]{rosenberg2020near}} Let $\bcc{X_k}^K_{k=1}$ be a martingale adapted to the filtration $\mathcal{F}_k$ which is almost surely bounded by $B$, i.e. $\abs{X_k} \leq B$ for all $k \in [K]$.
    Then, we have that with probability at least $1-\delta$
    $$
    \sum^K_{k=1} \mathbb{E}\bs{X_k|\mathcal{F}_k} \leq 2\sum^K_{k=1} X_k + 4 B\log(2K/\delta).
    $$
\end{lemma}

\begin{lemma}
    \textbf{Bernoulli anti-concentration \citep[Lemma 3]{cassel2025batch}} Let $L\geq 1$. Then, consider $\bcc{N^\ell}^L_{\ell=1}$ for any $N^\ell \geq 0$. Consider now samples $X_{n,\ell} \sim \mathrm{Bernoulli}(\mu)$ for each $\ell\in [L]$ and $n \in [N^\ell]$ and the estimators
    $$
    \hat{\mu}^\ell = \frac{\sum^{N^\ell}_{n=1}X_{n,\ell}}{N^\ell + 2}.
    $$
    Then, it holds that 
    $$
    \mathbb{P}\bs{\exists \ell \in [L]~~~\hat{\mu}^\ell \leq \mu} \geq 1 - e^{-2L/7}.
    $$
\end{lemma}

\begin{lemma}\textbf{Suboptimality to squared reward estimation error. \citep[Lemma C.2]{ji2026optimal}}
Let $r: \X\times\A \rightarrow \mathbb{R}$ be any reward function such that $r \leq r^\star$. Moreover, let $\pi^\star$ be the optimal policy for the regularized RL problem with reward $r^\star$ and $\pi_r$ the one when the reward is $r$. That is, 
$\pi^\star(a|x) \propto \piref(a|x) e^{r^\star(x,a)/\beta}$ and $\pi_r(a|x) \propto \piref(a|x) e^{r(x,a)/\beta}$. 
Then, it holds that
\[
J_\beta(\pi^\star) - J_\beta(\pi_r) \leq \beta^{-1} \sum_{x\in\X} \initial(x)\sum_{a\in \A} \pi^\star(a|x) (r(x,a) - r^\star(x,a))^2
\]
\end{lemma}
\begin{proof}
    We have that we can rewrite the suboptimality as follows
    \begin{align*}
        J_\beta(\pi^\star) - J_\beta(\pi_r) &= \innerprod{\pi^\star}{r^\star - \beta \log(\pi^\star/\piref)} - 
        \innerprod{\pi_r}{r^\star - \beta \log(\pi_r/\piref)} \\
        &= \beta \sum_{x\in \X} \initial(x) \log Z_{\star}(x) - \innerprod{\pi_r}{r^\star - r} - \beta \sum_{x\in \X} \initial(x)\log Z_r(x),
    \end{align*}
    where we have defined $Z_r(x) = \sum_{a\in \A} \piref(a|x) e^{r(x,a)/\beta}$ and the shorthand $Z_{\star} := Z_{r^\star}$.
    Now defining $$H(r, x) = \log Z_r(x) - \beta^{-1}\sum_{a\in\A} \pi_r(a|x) (r(x,a) - r^\star(x,a)),$$
    we can notice that
    \begin{align*}
        J_\beta(\pi^\star) - J_\beta(\pi_r) &= \beta \sum_{x\in \X} \initial(x) (H(r^\star, x) - H(r,x)) \\
        &=\beta \sum_{x\in \X} \initial(x)
        \sum_{a\in \A} (r^\star(x,a) - r(x,a)) \frac{ \partial H(r,x,a)}{\partial r(x,a)} \bigg |_{r = r_\gamma}
    \end{align*}
for $r_\gamma(x,a) = \gamma r(x,a) + (1-\gamma) r^\star(x,a)$ for some unknown $\gamma \in [0,1]$. Now, we compute $\nabla H(r_{\gamma},x,a)$.
\begin{align*}
    \frac{\partial H(r,x) }{\partial r (x,a)} \bigg |_{r=r_\gamma} &= \beta^{-1}\frac{\piref(a|x) e^{r_\gamma(x,a)/\beta}}{Z_{r_\gamma}(x)} - \beta^{-1}(r_\gamma(x,a)-r^\star(x,a)) \frac{\partial \pi_{r}(a|x)}{\partial r(x,a)} \\
    &- \beta^{-1}\pi_{r_\gamma}(a|x) \\
    & + \beta^{-1} \sum_{a'\neq a} (r_\gamma(x,a')-r^\star(x,a')) \frac{e^{r_\gamma(x,a')/\beta}}{Z^2_{r_\gamma}(x)} \frac{ \partial Z_r(x)}{\partial r(x,a)} \bigg|_{r=r_\gamma} \\
    &= - \beta^{-2}(r_\gamma(x,a)-r^\star(x,a)) \br{\pi_{r_\gamma}(a|x) - \pi^2_{r_\gamma}(a|x)} \\
    & + \beta^{-2} \sum_{a'\neq a} (r_\gamma(x,a')-r^\star(x,a')) \pi_{r_\gamma}(a'|x) \pi_{r_\gamma}(a|x),
\end{align*}
where we have used that $\frac{\partial Z_r(x)}{\partial r(x,a')} = \frac{\piref(a'|x)e^{r(x,a')/\beta}}{\beta}$ for any $a'\in \A$, $\frac{\partial \pi_r(a|x)}{\partial r(x,a)} = \beta^{-1} (\pi_r(a|x) - \pi^2_r(a|x) )$ for any $a \in \A$ and $\frac{\partial \pi_r(a|x)}{\partial r(x,a')} = -\beta^{-1}\pi_r(a|x) \pi_r(a'|x)$ for any $a\in \A$ and $a'\neq a$.
Then, we can rearrange the above expression as
\begin{align*}
    \frac{\partial H(r,x) }{\partial r (x,a)} \bigg |_{r=r_\gamma} &=- \beta^{-2}(r_\gamma(x,a)-r^\star(x,a)) \pi_{r_\gamma}(a|x) \\
    & + \beta^{-2} \sum_{a'\in \A} (r_\gamma(x,a')-r^\star(x,a')) \pi_{r_\gamma}(a'|x) \pi_{r_\gamma}(a|x)
    \\&\leq \beta^{-2}(r^\star(x,a) - r_\gamma(x,a)) \pi_{r_\gamma}(a|x)
\end{align*}
where the last step uses the fact that $r_\gamma - r^\star \leq 0$ because $r \leq r^\star$.
Then, we obtain that
\begin{align*}
    J_\beta(\pi^\star) - J_\beta(\pi_r)
        &\leq \beta^{-1} \underbrace{\sum_{x\in \X} \initial(x)
        \sum_{a\in \A} \pi_{r_\gamma}(a|x)(r^\star(x,a) - r(x,a))^2}_{f_\gamma},
\end{align*}
for some $\gamma \in [0,1]$. We conclude the proof by maximizing the above expression with respect to $\gamma$.
To this end, notice that
\begin{align*}
    \frac{\partial \pi_{r_\gamma}(a|x)}{\partial \gamma} &= \sum_{a'\in \A}\frac{\partial \pi_{r_\gamma}(a|x)}{\partial r_\gamma(x,a')} \frac{\partial r_\gamma (x,a')}{\partial \gamma} \\
    &= \sum_{a'\in \A}\frac{\partial \pi_{r_\gamma}(a|x)}{\partial r_\gamma (x,a')} (r^\star(x,a') - r(x,a')) \\
    &=\beta^{-1}(\pi_{r_\gamma}(a|x) - \pi^2_{r_\gamma}(a|x)) (r^\star(x,a) - r(x,a)) 
    \\&\phantom{=}-\beta^{-1}\sum_{a'\neq a} \pi_{r_\gamma}(a|x)\pi_{r_\gamma}(a'|x) (r^\star(x,a') - r(x,a')) \\
    & =\beta^{-1}\pi_{r_\gamma}(a|x) (r^\star(x,a) - r(x,a)) 
    \\&\phantom{=}-\beta^{-1}\sum_{a' \in \A} \pi_{r_\gamma}(a|x)\pi_{r_\gamma}(a'|x) (r^\star(x,a') - r(x,a')) 
\end{align*}
Therefore,
\begin{align*}
    \frac{\partial f_\gamma}{ \partial \gamma} &= \beta^{-1} \sum_{x\in \X}\initial(x)\sum_{a\in \A}  \pi_{r_\gamma}(a|x)(r^\star(x,a) - r(x,a))^3 \\
    &\phantom{=}-
    \beta^{-1} \sum_{x\in \X}\initial(x)\br{\sum_{a\in \A}  \pi_{r_\gamma}(a|x)(r^\star(x,a) - r(x,a))^2}\br{\sum_{a'\in \A}  \pi_{r_\gamma}(a'|x)(r^\star(x,a') - r(x,a'))}.
\end{align*}
Now, we prove that $\frac{\partial f_\gamma}{ \partial \gamma}\leq 0$, as consequence of the following more general statement,
$\mathbb{E}[X^3]  \leq \mathbb{E}[X^2]\mathbb{E}[X]$ for any random variable $X\geq 0$. To see this, we apply Jensen's inequality twice to obtain
\begin{align*}
    \mathbb{E}[X^3] = \mathbb{E}[(X^2)^{3/2}] \leq \mathbb{E}[X^2]^{3/2} = \mathbb{E}[X^2]\sqrt{\mathbb{E}[X^2]} \leq \mathbb{E}[X^2] \mathbb{E}[X],
\end{align*}
where the last inequality uses that $X\geq 0$. Then since $r^\star(x,a) - r(x,a) \geq 0 $, we obtain that $\frac{\partial f_\gamma}{ \partial \gamma} \leq 0$. Therefore, $f_\gamma$ is decreasing in $\gamma$, hence it is maximized for $\gamma = 0$. All in all, noticing that $\pi_{r_0} = \pi^\star$, we reach the conclusion of the proof, i.e.
\begin{align*}
    J_\beta(\pi^\star) - J_\beta(\pi_r)
        &\leq \beta^{-1} \sum_{x\in \X} \initial(x)
        \sum_{a\in \A} \pi^\star(a|x)(r^\star(x,a) - r(x,a))^2.
\end{align*}
\end{proof}
\subsubsection{Importance Sampling Routine: Proof of \Cref{lemma:rejsampling}}
\begin{proof}
Let us denote 
$$
f_{\mathrm{out}}(x,a) = \min_{\ell\in[L]} \pi^\ell(a|x) \exp(-\zeta^\ell(x))
$$
At each round, \Cref{alg:rejsampl} accepts the sample with probability
$$
\sum_{a\in\A} \pi_{\mathrm{prop}}(a|x) \frac{f_{\mathrm{out}}(x,a)}{\pi_{\mathrm{prop}}(a|x)} = \sum_{a\in \A} f_{\mathrm{out}}(x,a) = Z
$$
where we defined $Z=\sum_{a\in \A} \min_{\ell\in[L]} \pi^\ell(a|x) \exp(-\zeta^\ell(x))$ with $B=6e^{3\RMAX}$. We omitted the dependence on the prompt $x$ for simplicity.
Therefore, the random variable $X$  representing the number of trials until the first successful acceptance follows a Geometric distribution: $X \sim \text{Geometric}(Z)$ 

We want to find $N_{\RS}(x)$ such that the probability of failing to get a sample after $N_{\RS}(x)$ trials is at most $\delta$:

$$P(X > N_{\RS}(x)) \leq \delta$$

The event $X > N_{\RS}(x)$ occurs if and only if all $N_{\RS}(x)$ independent trials result in a rejection. The probability of rejection in a single trial is $(1 - Z)$. Therefore:

$$P(X > N_{\RS}(x)) = (1 - Z)^{N_{\RS}(x)}$$

Therefore, our goal can be rewritten as

$$(1 - Z)^{N_{\RS}(x)} \leq \delta$$

Taking the natural logarithm of both sides (and noting that $\ln(1-Z)$ is negative, which flips the inequality):

$$N_{\RS}(x) \ln(1 - Z) \leq \ln(\delta)$$
which implies 
$$N_{\RS}(x) \geq \frac{\ln(\delta)}{\ln(1 - Z)}$$

To simplify this into a more interpretable bound, we use the fact that for $Z \in (0, 1)$, $\ln(1 - Z) \leq -Z$. This implies that:

$$\frac{1}{\ln(1 - Z)} \geq -\frac{1}{Z}$$

Substituting this into our expression for $N_{\RS}(x)$:

$$N_{\RS}(x) \frac{\ln(\delta)}{- Z} = \frac{\ln(1/\delta)}{Z}$$.
\end{proof}
In virtue of \Cref{lemma:rejsampling}, will not output a valid sample with probability $\delta$. The maximum number of trials $N_{\RS}(x)$ is related by the above lemma to the desired confidence level. If the routine does not generate valid samples within $N_{\RS}(x)$ trials, then we can output \emph{I don't know}  as an answer to the prompt $x$.

We think that this is a very desirable property for LLM. It is important that these models can estimate their uncertainties even when they do not have access to a verifier.
This is a characteristic that it is not granted by the $\chi^2$ regularization system in \cite{huang2024correcting} but only by our ensemble mechanism when applied to DPO.
Surprisingly, the same feature does not appear if we consider reward ensemble based pessimism as done in \cite{coste2023reward}.

\end{document}

%% file: macros.tex
\usepackage{authblk}
\usepackage[utf8]{inputenc} 
\usepackage[T1]{fontenc}    
\usepackage{hyperref}       
\usepackage{url}            
\usepackage{booktabs}       
\usepackage{amsfonts}       
\usepackage{nicefrac}       
\usepackage{microtype}      
\usepackage{xcolor}         
\usepackage{titletoc}
\usepackage{algorithm}
\usepackage{algorithmic}
\usepackage{graphicx} 
\usepackage{natbib}
\usepackage{amsmath}
\usepackage{amssymb}
\usepackage{amsthm}
\usepackage{tikz}
\usetikzlibrary{automata,arrows, positioning}

\usepackage{subcaption} 

\usepackage{pifont}

\newcommand{\abs}[1]{\left| {#1} \right|}

\def\\X{{\mathcal{\X}}}
\def\A{{\mathcal{A}}}

\definecolor{greenp}{rgb}{0.0, 0.51, 0.5}

\newcommand{\initial}{\nu_0}
\newcommand{\bcc}[1]{\left\{{#1}\right\}}
\newcommand{\brr}[1]{\left({#1}\right)}
\newcommand{\bs}[1]{\left[{#1}\right]}
\newcommand{\norm}[1]{\left\| {#1} \right\|}

\newcommand{\innerprod}[2]{\left\langle{#1},{#2}\right\rangle}
\usepackage{dsfont}
\usepackage{thmtools}

\usepackage{mathtools} 
\newcommand{{\transpose}}{^\mathsf{\scriptscriptstyle T}}

\usepackage{mathtools}  

\newcommand{\cD}{\mathcal{D}}

\newcommand{\RS}{\mathrm{RS}}

\usepackage{tcolorbox}

\newcommand{\piout}{\pi_{\mathrm{out}}}

\newcommand{\X}{\mathcal{X}}

\newcommand{\RMAX}{R_{\max}}

\newcommand{\br}[1]{\left({#1}\right)}  
\newcommand{\ceil}[1]{\left\lceil #1 \right\rceil}  

\newcommand{\pess}{\mathrm{pess}}

\def\argmin{\mathop{\mbox{ arg\,min}}}
\def\argmax{\mathop{\mbox{ arg\,max}}}

\newcommand{\ptie}{\mathfrak{p}_{\mathrm{tie}}}

\renewcommand{\phi}{\varphi}

\newcommand{\piref}{\pi_{\mathrm{ref}}}
\newcommand{\pidata}{\pi_{\mathrm{data}}}

\usepackage{pifont}

\usepackage{amsmath} 
\usepackage{booktabs} 
\usepackage{makecell}

\usepackage{pgfplots}
\usepgfplotslibrary{groupplots}
\usetikzlibrary{calc} 
\pgfplotsset{compat=1.17}
\usepackage{wrapfig}

%% file: references.bib
@article{ji2026optimal,
  title={On the Optimal Sample Complexity of Offline Multi-Armed Bandits with KL Regularization},
  author={Ji, Kaixuan and Di, Qiwei and Zhao, Heyang and Zhao, Qingyue and Gu, Quanquan},
  journal={arXiv preprint arXiv:2605.02141},
  year={2026}
}

@article{liu2024provably,
  title={Provably mitigating overoptimization in rlhf: Your sft loss is implicitly an adversarial regularizer},
  author={Liu, Zhihan and Lu, Miao and Zhang, Shenao and Liu, Boyi and Guo, Hongyi and Yang, Yingxiang and Blanchet, Jose and Wang, Zhaoran},
  journal={Advances in Neural Information Processing Systems},
  volume={37},
  pages={138663--138697},
  year={2024}
}

@inproceedings{cassel2025batch,
  title={Batch ensemble for variance dependent regret in stochastic bandits},
  author={Cassel, Asaf and Levy, Orin and Mansour, Yishay},
  booktitle={Proceedings of the AAAI Conference on Artificial Intelligence},
  volume={39},
  pages={15678--15685},
  year={2025}
}

@inproceedings{Abbasi-Yadkori:2011,
 author = {Abbasi-Yadkori, Yasin and P\'{a}l, D\'{a}vid and Szepesv\'{a}ri, Csaba},
 title = {Improved Algorithms for Linear Stochastic Bandits},
 booktitle = {Advances in Neural Information Processing Systems (NeurIPS)},
 year = {2011},
}

@inproceedings{
viano2024imitation,
title={Imitation Learning in Discounted Linear {MDP}s without exploration assumptions},
author={Luca Viano and Stratis Skoulakis and Volkan Cevher},
booktitle={Forty-first International Conference on Machine Learning},
year={2024},
url={https://openreview.net/forum?id=DChQpB4AJy}
}

@article{cen2024value,
  title={Value-incentivized preference optimization: A unified approach to online and offline rlhf},
  author={Cen, Shicong and Mei, Jincheng and Goshvadi, Katayoon and Dai, Hanjun and Yang, Tong and Yang, Sherry and Schuurmans, Dale and Chi, Yuejie and Dai, Bo},
  journal={arXiv preprint arXiv:2405.19320},
  year={2024}
}

@article{li2023reinforcement,
  title={Reinforcement learning with human feedback: Learning dynamic choices via pessimism},
  author={Li, Zihao and Yang, Zhuoran and Wang, Mengdi},
  journal={arXiv preprint arXiv:2305.18438},
  year={2023}
}

@article{rafailov2023direct,
  title={Direct preference optimization: Your language model is secretly a reward model},
  author={Rafailov, Rafael and Sharma, Archit and Mitchell, Eric and Manning, Christopher D and Ermon, Stefano and Finn, Chelsea},
  journal={Advances in Neural Information Processing Systems},
  volume={36},
  pages={53728--53741},
  year={2023}
}

@article{fisch2024robust,
  title={Robust preference optimization through reward model distillation},
  author={Fisch, Adam and Eisenstein, Jacob and Zayats, Vicky and Agarwal, Alekh and Beirami, Ahmad and Nagpal, Chirag and Shaw, Pete and Berant, Jonathan},
  journal={arXiv preprint arXiv:2405.19316},
  year={2024}
}

@article{ji2024self,
  title={Self-play with adversarial critic: Provable and scalable offline alignment for language models},
  author={Ji, Xiang and Kulkarni, Sanjeev and Wang, Mengdi and Xie, Tengyang},
  journal={arXiv preprint arXiv:2406.04274},
  year={2024}
}

@article{schlaginhaufen2025efficient,
  title={Efficient Preference-Based Reinforcement Learning: Randomized Exploration Meets Experimental Design},
  author={Schlaginhaufen, Andreas and Ouhamma, Reda and Kamgarpour, Maryam},
  journal={arXiv preprint arXiv:2506.09508},
  year={2025}
}

@article{zhan2023provable,
  title={Provable offline preference-based reinforcement learning},
  author={Zhan, Wenhao and Uehara, Masatoshi and Kallus, Nathan and Lee, Jason D and Sun, Wen},
  journal={arXiv preprint arXiv:2305.14816},
  year={2023}
}

@inproceedings{zhu2023principled,
  title={Principled reinforcement learning with human feedback from pairwise or k-wise comparisons},
  author={Zhu, Banghua and Jordan, Michael and Jiao, Jiantao},
  booktitle={International Conference on Machine Learning},
  pages={43037--43067},
  year={2023},
  organization={PMLR}
}

@misc{cui2023ultrafeedback,
      title={UltraFeedback: Boosting Language Models with High-quality Feedback}, 
      author={Ganqu Cui and Lifan Yuan and Ning Ding and Guanming Yao and Wei Zhu and Yuan Ni and Guotong Xie and Zhiyuan Liu and Maosong Sun},
      year={2023},
      eprint={2310.01377},
      archivePrefix={arXiv},
      primaryClass={cs.CL}
}

@article{viel2025soar,
  title={IL-SOAR: Imitation Learning with Soft Optimistic Actor cRitic},
  author={Viel, Stefano and Viano, Luca and Cevher, Volkan},
  journal={arXiv preprint arXiv:2502.19859},
  year={2025}
}

@InProceedings{rosenberg2020near,
  title = 	 {Near-optimal Regret Bounds for Stochastic Shortest Path},
  author =       {Rosenberg, Aviv and Cohen, Alon and Mansour, Yishay and Kaplan, Haim},
  booktitle = 	 {Proceedings of the 37th International Conference on Machine Learning},
  pages = 	 {8210--8219},
  year = 	 {2020},
  editor = 	 {III, Hal Daumé and Singh, Aarti},
  volume = 	 {119},
  series = 	 {Proceedings of Machine Learning Research},
  month = 	 {13--18 Jul},
  publisher =    {PMLR},
  url = 	 {https://proceedings.mlr.press/v119/rosenberg20a.html}
}

@article{matrenok2025quantile,
  title={Quantile Reward Policy Optimization: Alignment with Pointwise Regression and Exact Partition Functions},
  author={Matrenok, Simon and Moalla, Skander and Gulcehre, Caglar},
  journal={arXiv preprint arXiv:2507.08068},
  year={2025}
}

@inproceedings{hu2022lora,
  title={{LoRA}: Low-Rank Adaptation of Large Language Models},
  author={Hu, Edward J. and Shen, Yelong and Wallis, Phillip and Allen-Zhu, Zeyuan and Li, Yuanzhi and Wang, Shean and Wang, Lu and Chen, Weizhu},
  booktitle={International Conference on Learning Representations},
  year={2022},
  url={https://openreview.net/forum?id=nZeVKeeFYf9}
}

@article{Shani:2021,
      title={Online Apprenticeship Learning}, 
      author={Lior Shani and Tom Zahavy and Shie Mannor},
      year={2021},
      journal={arXiv:2102.06924},
}

@inproceedings{park2024disentangling,
  title={Disentangling Length from Quality in Direct Preference Optimization},
  author={Park, Ryan and Rafailov, Rafael and Ermon, Stefano and Finn, Chelsea},
  booktitle={Findings of the Association for Computational Linguistics: ACL 2024},
  pages={4998--5017},
  year={2024},
  publisher={Association for Computational Linguistics}
}

@inproceedings{gao2023scaling,
  title={Scaling Laws for Reward Model Overoptimization},
  author={Gao, Leo and Schulman, John and Hilton, Jacob},
  booktitle={International Conference on Machine Learning},
  pages={10835--10866},
  year={2023},
  organization={PMLR}
}

@article{gorbatovski2024learn,
  title={Learn your reference model for real good alignment},
  author={Gorbatovski, Alexey and Shaposhnikov, Boris and Malakhov, Alexey and Surnachev, Nikita and Aksenov, Yaroslav and Maksimov, Ian and Balagansky, Nikita and Gavrilov, Daniil},
  journal={arXiv preprint arXiv:2404.09656},
  year={2024}
}

@article{xu2024is,
  title={Is DPO Superior to PPO for LLM Alignment? A Comprehensive Study},
  author={Xu, Shusheng and Fu, Wei and Gao, Jiaxuan and Ye, Wenjie and Liu, Weiling and Mei, Zhiyu and Wang, Guangju and Yu, Chao and Wu, Yi},
  journal={arXiv preprint arXiv:2404.10719},
  year={2024}
}

@article{chen2025avoiding,
  title={Avoiding $\mathcal{O}(e^{R_{\max}})$ scaling in RLHF through Preference-based Exploration},
  author={Chen, Mingyu and Chen, Yiding and Sun, Wen and Zhang, Xuezhou},
  journal={arXiv preprint arXiv:2502.00666},
  year={2025}
}

@article{tuyls2025representation,
  title={Representation-Based Exploration for Language Models: From Test-Time to Post-Training},
  author={Tuyls, Jens and Foster, Dylan J and Krishnamurthy, Akshay and Ash, Jordan T},
  journal={arXiv preprint arXiv:2510.11686},
  year={2025}
}

@misc{deepseekai2025deepseekr1,
      title={DeepSeek-R1: Incentivizing Reasoning Capability in LLMs via Reinforcement Learning}, 
      author={DeepSeek-AI and Daya Guo and Dejian Yang and Haowei Zhang and Junxiao Song and Ruibin Yuan and Qihao Zhu and Wanghai Xu and Wentao Li and Y. K. Li and others},
      year={2025},
      eprint={2501.12948},
      archivePrefix={arXiv},
      primaryClass={cs.CL},
      url={https://arxiv.org/abs/2501.12948}, 
}

@article{bamba2025xrpo,
  title={XRPO: Pushing the limits of GRPO with Targeted Exploration and Exploitation},
  author={Bamba, Udbhav and Fang, Minghao and Yu, Yifan and Zheng, Haizhong and Lai, Fan},
  journal={arXiv preprint arXiv:2510.06672},
  year={2025}
}

@inproceedings{ishfaq2021randomized,
  title={Randomized exploration in reinforcement learning with general value function approximation},
  author={Ishfaq, Haque and Cui, Qiwen and Nguyen, Viet and Ayoub, Alex and Yang, Zhuoran and Wang, Zhaoran and Precup, Doina and Yang, Lin},
  booktitle={International Conference on Machine Learning},
  pages={4607--4616},
  year={2021},
  organization={PMLR}
}

@article{ishfaq2023provable,
  title={Provable and practical: Efficient exploration in reinforcement learning via langevin monte carlo},
  author={Ishfaq, Haque and Lan, Qingfeng and Xu, Pan and Mahmood, A Rupam and Precup, Doina and Anandkumar, Anima and Azizzadenesheli, Kamyar},
  journal={arXiv preprint arXiv:2305.18246},
  year={2023}
}

@article{ishfaq2025langevin,
  title={Langevin Soft Actor-Critic: Efficient Exploration through Uncertainty-Driven Critic Learning},
  author={Ishfaq, Haque and Wang, Guangyuan and Islam, Sami Nur and Precup, Doina},
  journal={arXiv preprint arXiv:2501.17827},
  year={2025}
}

@article{osband2014near,
  title={Near-optimal reinforcement learning in factored mdps},
  author={Osband, Ian and Van Roy, Benjamin},
  journal={Advances in Neural Information Processing Systems},
  volume={27},
  year={2014}
}

@inproceedings{osband2016generalization,
  title={Generalization and exploration via randomized value functions},
  author={Osband, Ian and Van Roy, Benjamin and Wen, Zheng},
  booktitle={International Conference on Machine Learning},
  pages={2377--2386},
  year={2016},
  organization={PMLR}
}

@article{osband2023epistemic,
  title={Epistemic neural networks},
  author={Osband, Ian and Wen, Zheng and Asghari, Seyed Mohammad and Dwaracherla, Vikranth and Ibrahimi, Morteza and Lu, Xiuyuan and Van Roy, Benjamin},
  journal={Advances in Neural Information Processing Systems},
  volume={36},
  pages={2795--2823},
  year={2023}
}

@inproceedings{osband2023approximate,
  title={Approximate thompson sampling via epistemic neural networks},
  author={Osband, Ian and Wen, Zheng and Asghari, Seyed Mohammad and Dwaracherla, Vikranth and Ibrahimi, Morteza and Lu, Xiuyuan and Van Roy, Benjamin},
  booktitle={Uncertainty in Artificial Intelligence},
  pages={1586--1595},
  year={2023},
  organization={PMLR}
}

@article{osband2018randomized,
  title={Randomized prior functions for deep reinforcement learning},
  author={Osband, Ian and Aslanides, John and Cassirer, Albin},
  journal={Advances in Neural Information Processing Systems},
  volume={31},
  year={2018}
}

@article{osband2016deep,
  title={Deep exploration via bootstrapped DQN},
  author={Osband, Ian and Blundell, Charles and Pritzel, Alexander and Van Roy, Benjamin},
  journal={Advances in neural information processing systems},
  volume={29},
  year={2016}
}

@article{o2021variational,
  title={Variational bayesian reinforcement learning with regret bounds},
  author={O'Donoghue, Brendan},
  journal={Advances in Neural Information Processing Systems},
  volume={34},
  pages={28208--28221},
  year={2021}
}

@article{tarbouriech2024probabilistic,
  title={Probabilistic inference in reinforcement learning done right},
  author={Tarbouriech, Jean and Lattimore, Tor and O'Donoghue, Brendan},
  journal={Advances in Neural Information Processing Systems},
  volume={36},
  year={2024}
}

@article{o2023efficient,
  title={Efficient exploration via epistemic-risk-seeking policy optimization},
  author={O'Donoghue, Brendan},
  journal={arXiv preprint arXiv:2302.09339},
  year={2023}
}

@article{chen2017ucb,
  title={Ucb exploration via q-ensembles},
  author={Chen, Richard Y and Sidor, Szymon and Abbeel, Pieter and Schulman, John},
  journal={arXiv preprint arXiv:1706.01502},
  year={2017}
}

@article{ishfaq2024more,
  title={More efficient randomized exploration for reinforcement learning via approximate sampling},
  author={Ishfaq, Haque and Tan, Yixin and Yang, Yu and Lan, Qingfeng and Lu, Jianfeng and Mahmood, A Rupam and Precup, Doina and Xu, Pan},
  journal={arXiv preprint arXiv:2406.12241},
  year={2024}
}

@article{ding2025multi,
  title={Multi-Layer GRPO: Enhancing Reasoning and Self-Correction in Large Language Models},
  author={Ding, Fei and Wang, Baiqiao and Zeng, Zijian and Wang, Youwei},
  journal={arXiv preprint arXiv:2506.04746},
  year={2025}
}

@inproceedings{munos2024nash,
  title={Nash Learning from Human Feedback},
  author={Munos, R{\'e}mi and Valko, Michal and Calandriello, Daniele and Gheshlaghi Azar, Mohammad and Rowland, Mark and Guo, Zhaohan Daniel and Tang, Yunhao and Geist, Matthieu and Mesnard, Thomas and Fiegel, C{\^o}me and Michi, Andrea and Selvi, Marco and Girgin, Sertan and Momchev, Nikola and Bachem, Olivier and Mankowitz, Daniel J and Precup, Doina and Piot, Bilal},
  booktitle={Proceedings of the 41st International Conference on Machine Learning (ICML)},
  year={2024},
  url={https://arxiv.org/abs/2312.00886}
}

@misc{guo2024direct,
      title={Direct Language Model Alignment from Online AI Feedback}, 
      author={Shangmin Guo and Biao Zhang and Tianlin Liu and Tianqi Liu and Misha Khalman and Felipe Llinares and Alexandre Rame and Thomas Mesnard and Yao Zhao and Bilal Piot and Johan Ferret and Mathieu Blondel},
      year={2024},
      eprint={2402.04792},
      archivePrefix={arXiv},
      primaryClass={cs.LG},
      url={https://arxiv.org/abs/2402.04792}, 
}

@article{yuan2025trajectory,
  title={Trajectory Bellman Residual Minimization: A Simple Value-Based Method for LLM Reasoning},
  author={Yuan, Yurun and Chen, Fan and Jia, Zeyu and Rakhlin, Alexander and Xie, Tengyang},
  journal={arXiv preprint arXiv:2505.15311},
  year={2025}
}

@misc{liu2025understanding,
      title={Understanding R1-Zero-Like Training: A Critical Perspective}, 
      author={Zichen Liu and others},
      year={2025},
      eprint={2503.20783},
      archivePrefix={arXiv},
      primaryClass={cs.LG},
      url={https://arxiv.org/abs/2503.20783},
}

@misc{shao2024deepseekmath,
      title={DeepSeekMath: Pushing the Limits of Mathematical Reasoning in Open Language Models}, 
      author={Zhihong Shao and Peiyi Wang and Qihao Zhu and Runxin Xu and Junxiao Song and Xiao Bi and Haowei Zhang and Mingchuan Zhang and Y.K. Li and Y. Wu and Daya Guo},
      year={2024},
      eprint={2402.03300},
      archivePrefix={arXiv},
      primaryClass={cs.CL},
      url={https://arxiv.org/abs/2402.03300}, 
}

@article{agarwal2025design,
  title={Design Considerations in Offline Preference-based RL},
  author={Agarwal, Alekh and Dann, Christoph and Marinov, Teodor V},
  journal={arXiv preprint arXiv:2502.06861},
  year={2025}
}

@article{huang2024correcting,
  title={Correcting the mythos of kl-regularization: Direct alignment without overoptimization via chi-squared preference optimization},
  author={Huang, Audrey and Zhan, Wenhao and Xie, Tengyang and Lee, Jason D and Sun, Wen and Krishnamurthy, Akshay and Foster, Dylan J},
  journal={arXiv preprint arXiv:2407.13399},
  year={2024}
}

@article{song2024importance,
  title={The importance of online data: Understanding preference fine-tuning via coverage},
  author={Song, Yuda and Swamy, Gokul and Singh, Aarti and Bagnell, J and Sun, Wen},
  journal={Advances in Neural Information Processing Systems},
  volume={37},
  pages={12243--12270},
  year={2024}
}

@inproceedings{azar2024general,
  title={A general theoretical paradigm to understand learning from human preferences},
  author={Azar, Mohammad Gheshlaghi and Guo, Zhaohan Daniel and Piot, Bilal and Munos, Remi and Rowland, Mark and Valko, Michal and Calandriello, Daniele},
  booktitle={International Conference on Artificial Intelligence and Statistics},
  pages={4447--4455},
  year={2024},
  organization={PMLR}
}

@article{coste2023reward,
  title={Reward model ensembles help mitigate overoptimization},
  author={Coste, Thomas and Anwar, Usman and Kirk, Robert and Krueger, David},
  journal={arXiv preprint arXiv:2310.02743},
  year={2023}
}

@article{meng2024simpo,
  title={Simpo: Simple preference optimization with a reference-free reward},
  author={Meng, Yu and Xia, Mengzhou and Chen, Danqi},
  journal={Advances in Neural Information Processing Systems},
  volume={37},
  pages={124198--124235},
  year={2024}
}

@misc{alpaca_eval,
  author = {Xuechen Li and Tianyi Zhang and Yann Dubois and Rohan Taori and Ishaan Gulrajani and Carlos Guestrin and Percy Liang and Tatsunori B. Hashimoto },
  title = {AlpacaEval: An Automatic Evaluator of Instruction-following Models},
  year = {2023},
  month = {5},
  publisher = {GitHub},
  journal = {GitHub repository},
  howpublished = {\url{https://github.com/tatsu-lab/alpaca_eval}}
}

@article{christiano2017deep,
  title={Deep reinforcement learning from human preferences},
  author={Christiano, Paul F and Leike, Jan and Brown, Tom and Martic, Miljan and Legg, Shane and Amodei, Dario},
  journal={Advances in neural information processing systems},
  volume={30},
  year={2017}
}

@article{ouyang2022training,
  title={Training language models to follow instructions with human feedback},
  author={Ouyang, Long and Wu, Jeffrey and Jiang, Xu and Almeida, Diogo and Wainwright, Carroll and Mishkin, Pamela and Zhang, Chong and Agarwal, Sandhini and Slama, Katarina and Ray, Alex and others},
  journal={Advances in neural information processing systems},
  volume={35},
  pages={27730--27744},
  year={2022}
}

@article{yadkori2024believe,
  title={To believe or not to believe your llm},
  author={Yadkori, Yasin Abbasi and Kuzborskij, Ilja and Gy{\"o}rgy, Andr{\'a}s and Szepesv{\'a}ri, Csaba},
  journal={arXiv preprint arXiv:2406.02543},
  year={2024}
}

@article{das2024active,
  title={Active preference optimization for sample efficient rlhf},
  author={Das, Nirjhar and Chakraborty, Souradip and Pacchiano, Aldo and Chowdhury, Sayak Ray},
  journal={arXiv preprint arXiv:2402.10500},
  year={2024}
}

@article{xie2024exploratory,
  title={Exploratory preference optimization: Harnessing implicit q*-approximation for sample-efficient rlhf},
  author={Xie, Tengyang and Foster, Dylan J and Krishnamurthy, Akshay and Rosset, Corby and Awadallah, Ahmed and Rakhlin, Alexander},
  journal={arXiv preprint arXiv:2405.21046},
  year={2024}
}

@article{moulin2025optimistically,
  title={Optimistically Optimistic Exploration for Provably Efficient Infinite-Horizon Reinforcement and Imitation Learning},
  author={Moulin, Antoine and Neu, Gergely and Viano, Luca},
  journal={arXiv preprint arXiv:2502.13900},
  year={2025}
}
